\documentclass[twoside,11pt]{article}

\usepackage{bbm}
\usepackage{jmlr2e}
\usepackage{amsmath}
\usepackage{xcolor}
\usepackage{bm}
\usepackage{tikz}
\usetikzlibrary{bayesnet}
\usepackage{subcaption}
\usepackage{float}
\usepackage{stackengine}
\usepackage[draft]{minted}
\usepackage{array}
\usepackage{multirow}
\usepackage{stmaryrd}
\usepackage{pgfplots}
\usepgfplotslibrary{fillbetween}
\usetikzlibrary{patterns}
\usetikzlibrary{shapes.geometric}
\usepackage[ruled,vlined]{algorithm2e}
\usepackage[margin=0.5in]{geometry}

\pgfmathdeclarefunction{gauss}{3}{%
  \pgfmathparse{1/(#3*sqrt(2*pi))*exp(-((#1-#2)^2)/(2*#3^2))}%
}
\pgfmathdeclarefunction{sum_gauss}{5}{%
  \pgfmathparse{1/(3*#3*sqrt(2*pi))*exp(-((#1-#2)^2)/(2*#3^2)) + 2/(3*#5*sqrt(2*pi))*exp(-((#1-#4)^2)/(2*#5^2))}%
}
\newcommand*\IdMLast[1]{{\normalfont #1\hspace{-0.04cm}\setminus\hspace{-0.04cm} last}}

\newcommand{\kl}[2]{D_{\mathrm{KL}} \left[ \left. \left. #1 \right|\right| #2 \right] }
\DeclareMathOperator*{\argmax}{arg\,max}

\DeclareMathOperator*{\argmins}{arg\,mins}
\definecolor{blue-violet}{rgb}{0.54, 0.17, 0.89}
\definecolor{bluepigment}{rgb}{0.2, 0.2, 0.6}
\makeatletter
\newcommand*\bigcdot{\mathpalette\bigcdot@{.5}}
\newcommand*\bigcdot@[2]{\mathbin{\vcenter{\hbox{\scalebox{#2}{$\m@th#1\bullet$}}}}}
\makeatother
\def\delequal{\mathrel{\ensurestackMath{\stackon[1pt]{=}{\scriptstyle\Delta}}}}
\definecolor{Yellow}{RGB}{211, 176, 15}
\definecolor{Green}{rgb}{0.01, 0.75, 0.24}
\colorlet{Red}{red!50!black}
\colorlet{Blue}{blue!50!black}
\definecolor{Violet}{rgb}{0.56,0.14,0.56}

\makeatletter
\def\PYGdefault@reset{\let\PYGdefault@it=\relax \let\PYGdefault@bf=\relax%
    \let\PYGdefault@ul=\relax \let\PYGdefault@tc=\relax%
    \let\PYGdefault@bc=\relax \let\PYGdefault@ff=\relax}
\def\PYGdefault@tok#1{\csname PYGdefault@tok@#1\endcsname}
\def\PYGdefault@toks#1+{\ifx\relax#1\empty\else%
    \PYGdefault@tok{#1}\expandafter\PYGdefault@toks\fi}
\def\PYGdefault@do#1{\PYGdefault@bc{\PYGdefault@tc{\PYGdefault@ul{%
    \PYGdefault@it{\PYGdefault@bf{\PYGdefault@ff{#1}}}}}}}
\def\PYGdefault#1#2{\PYGdefault@reset\PYGdefault@toks#1+\relax+\PYGdefault@do{#2}}

\expandafter\def\csname PYGdefault@tok@w\endcsname{\def\PYGdefault@tc##1{\textcolor[rgb]{0.73,0.73,0.73}{##1}}}
\expandafter\def\csname PYGdefault@tok@c\endcsname{\let\PYGdefault@it=\textit\def\PYGdefault@tc##1{\textcolor[rgb]{0.25,0.50,0.50}{##1}}}
\expandafter\def\csname PYGdefault@tok@cp\endcsname{\def\PYGdefault@tc##1{\textcolor[rgb]{0.74,0.48,0.00}{##1}}}
\expandafter\def\csname PYGdefault@tok@k\endcsname{\let\PYGdefault@bf=\textbf\def\PYGdefault@tc##1{\textcolor[rgb]{0.00,0.50,0.00}{##1}}}
\expandafter\def\csname PYGdefault@tok@kp\endcsname{\def\PYGdefault@tc##1{\textcolor[rgb]{0.00,0.50,0.00}{##1}}}
\expandafter\def\csname PYGdefault@tok@kt\endcsname{\def\PYGdefault@tc##1{\textcolor[rgb]{0.69,0.00,0.25}{##1}}}
\expandafter\def\csname PYGdefault@tok@o\endcsname{\def\PYGdefault@tc##1{\textcolor[rgb]{0.40,0.40,0.40}{##1}}}
\expandafter\def\csname PYGdefault@tok@ow\endcsname{\let\PYGdefault@bf=\textbf\def\PYGdefault@tc##1{\textcolor[rgb]{0.67,0.13,1.00}{##1}}}
\expandafter\def\csname PYGdefault@tok@nb\endcsname{\def\PYGdefault@tc##1{\textcolor[rgb]{0.00,0.50,0.00}{##1}}}
\expandafter\def\csname PYGdefault@tok@nf\endcsname{\def\PYGdefault@tc##1{\textcolor[rgb]{0.00,0.00,1.00}{##1}}}
\expandafter\def\csname PYGdefault@tok@nc\endcsname{\let\PYGdefault@bf=\textbf\def\PYGdefault@tc##1{\textcolor[rgb]{0.00,0.00,1.00}{##1}}}
\expandafter\def\csname PYGdefault@tok@nn\endcsname{\let\PYGdefault@bf=\textbf\def\PYGdefault@tc##1{\textcolor[rgb]{0.00,0.00,1.00}{##1}}}
\expandafter\def\csname PYGdefault@tok@ne\endcsname{\let\PYGdefault@bf=\textbf\def\PYGdefault@tc##1{\textcolor[rgb]{0.82,0.25,0.23}{##1}}}
\expandafter\def\csname PYGdefault@tok@nv\endcsname{\def\PYGdefault@tc##1{\textcolor[rgb]{0.10,0.09,0.49}{##1}}}
\expandafter\def\csname PYGdefault@tok@no\endcsname{\def\PYGdefault@tc##1{\textcolor[rgb]{0.53,0.00,0.00}{##1}}}
\expandafter\def\csname PYGdefault@tok@nl\endcsname{\def\PYGdefault@tc##1{\textcolor[rgb]{0.63,0.63,0.00}{##1}}}
\expandafter\def\csname PYGdefault@tok@ni\endcsname{\let\PYGdefault@bf=\textbf\def\PYGdefault@tc##1{\textcolor[rgb]{0.60,0.60,0.60}{##1}}}
\expandafter\def\csname PYGdefault@tok@na\endcsname{\def\PYGdefault@tc##1{\textcolor[rgb]{0.49,0.56,0.16}{##1}}}
\expandafter\def\csname PYGdefault@tok@nt\endcsname{\let\PYGdefault@bf=\textbf\def\PYGdefault@tc##1{\textcolor[rgb]{0.00,0.50,0.00}{##1}}}
\expandafter\def\csname PYGdefault@tok@nd\endcsname{\def\PYGdefault@tc##1{\textcolor[rgb]{0.67,0.13,1.00}{##1}}}
\expandafter\def\csname PYGdefault@tok@s\endcsname{\def\PYGdefault@tc##1{\textcolor[rgb]{0.73,0.13,0.13}{##1}}}
\expandafter\def\csname PYGdefault@tok@sd\endcsname{\let\PYGdefault@it=\textit\def\PYGdefault@tc##1{\textcolor[rgb]{0.73,0.13,0.13}{##1}}}
\expandafter\def\csname PYGdefault@tok@si\endcsname{\let\PYGdefault@bf=\textbf\def\PYGdefault@tc##1{\textcolor[rgb]{0.73,0.40,0.53}{##1}}}
\expandafter\def\csname PYGdefault@tok@se\endcsname{\let\PYGdefault@bf=\textbf\def\PYGdefault@tc##1{\textcolor[rgb]{0.73,0.40,0.13}{##1}}}
\expandafter\def\csname PYGdefault@tok@sr\endcsname{\def\PYGdefault@tc##1{\textcolor[rgb]{0.73,0.40,0.53}{##1}}}
\expandafter\def\csname PYGdefault@tok@ss\endcsname{\def\PYGdefault@tc##1{\textcolor[rgb]{0.10,0.09,0.49}{##1}}}
\expandafter\def\csname PYGdefault@tok@sx\endcsname{\def\PYGdefault@tc##1{\textcolor[rgb]{0.00,0.50,0.00}{##1}}}
\expandafter\def\csname PYGdefault@tok@m\endcsname{\def\PYGdefault@tc##1{\textcolor[rgb]{0.40,0.40,0.40}{##1}}}
\expandafter\def\csname PYGdefault@tok@gh\endcsname{\let\PYGdefault@bf=\textbf\def\PYGdefault@tc##1{\textcolor[rgb]{0.00,0.00,0.50}{##1}}}
\expandafter\def\csname PYGdefault@tok@gu\endcsname{\let\PYGdefault@bf=\textbf\def\PYGdefault@tc##1{\textcolor[rgb]{0.50,0.00,0.50}{##1}}}
\expandafter\def\csname PYGdefault@tok@gd\endcsname{\def\PYGdefault@tc##1{\textcolor[rgb]{0.63,0.00,0.00}{##1}}}
\expandafter\def\csname PYGdefault@tok@gi\endcsname{\def\PYGdefault@tc##1{\textcolor[rgb]{0.00,0.63,0.00}{##1}}}
\expandafter\def\csname PYGdefault@tok@gr\endcsname{\def\PYGdefault@tc##1{\textcolor[rgb]{1.00,0.00,0.00}{##1}}}
\expandafter\def\csname PYGdefault@tok@ge\endcsname{\let\PYGdefault@it=\textit}
\expandafter\def\csname PYGdefault@tok@gs\endcsname{\let\PYGdefault@bf=\textbf}
\expandafter\def\csname PYGdefault@tok@gp\endcsname{\let\PYGdefault@bf=\textbf\def\PYGdefault@tc##1{\textcolor[rgb]{0.00,0.00,0.50}{##1}}}
\expandafter\def\csname PYGdefault@tok@go\endcsname{\def\PYGdefault@tc##1{\textcolor[rgb]{0.53,0.53,0.53}{##1}}}
\expandafter\def\csname PYGdefault@tok@gt\endcsname{\def\PYGdefault@tc##1{\textcolor[rgb]{0.00,0.27,0.87}{##1}}}
\expandafter\def\csname PYGdefault@tok@err\endcsname{\def\PYGdefault@bc##1{\setlength{\fboxsep}{0pt}\fcolorbox[rgb]{1.00,0.00,0.00}{1,1,1}{\strut ##1}}}
\expandafter\def\csname PYGdefault@tok@kc\endcsname{\let\PYGdefault@bf=\textbf\def\PYGdefault@tc##1{\textcolor[rgb]{0.00,0.50,0.00}{##1}}}
\expandafter\def\csname PYGdefault@tok@kd\endcsname{\let\PYGdefault@bf=\textbf\def\PYGdefault@tc##1{\textcolor[rgb]{0.00,0.50,0.00}{##1}}}
\expandafter\def\csname PYGdefault@tok@kn\endcsname{\let\PYGdefault@bf=\textbf\def\PYGdefault@tc##1{\textcolor[rgb]{0.00,0.50,0.00}{##1}}}
\expandafter\def\csname PYGdefault@tok@kr\endcsname{\let\PYGdefault@bf=\textbf\def\PYGdefault@tc##1{\textcolor[rgb]{0.00,0.50,0.00}{##1}}}
\expandafter\def\csname PYGdefault@tok@bp\endcsname{\def\PYGdefault@tc##1{\textcolor[rgb]{0.00,0.50,0.00}{##1}}}
\expandafter\def\csname PYGdefault@tok@fm\endcsname{\def\PYGdefault@tc##1{\textcolor[rgb]{0.00,0.00,1.00}{##1}}}
\expandafter\def\csname PYGdefault@tok@vc\endcsname{\def\PYGdefault@tc##1{\textcolor[rgb]{0.10,0.09,0.49}{##1}}}
\expandafter\def\csname PYGdefault@tok@vg\endcsname{\def\PYGdefault@tc##1{\textcolor[rgb]{0.10,0.09,0.49}{##1}}}
\expandafter\def\csname PYGdefault@tok@vi\endcsname{\def\PYGdefault@tc##1{\textcolor[rgb]{0.10,0.09,0.49}{##1}}}
\expandafter\def\csname PYGdefault@tok@vm\endcsname{\def\PYGdefault@tc##1{\textcolor[rgb]{0.10,0.09,0.49}{##1}}}
\expandafter\def\csname PYGdefault@tok@sa\endcsname{\def\PYGdefault@tc##1{\textcolor[rgb]{0.73,0.13,0.13}{##1}}}
\expandafter\def\csname PYGdefault@tok@sb\endcsname{\def\PYGdefault@tc##1{\textcolor[rgb]{0.73,0.13,0.13}{##1}}}
\expandafter\def\csname PYGdefault@tok@sc\endcsname{\def\PYGdefault@tc##1{\textcolor[rgb]{0.73,0.13,0.13}{##1}}}
\expandafter\def\csname PYGdefault@tok@dl\endcsname{\def\PYGdefault@tc##1{\textcolor[rgb]{0.73,0.13,0.13}{##1}}}
\expandafter\def\csname PYGdefault@tok@s2\endcsname{\def\PYGdefault@tc##1{\textcolor[rgb]{0.73,0.13,0.13}{##1}}}
\expandafter\def\csname PYGdefault@tok@sh\endcsname{\def\PYGdefault@tc##1{\textcolor[rgb]{0.73,0.13,0.13}{##1}}}
\expandafter\def\csname PYGdefault@tok@s1\endcsname{\def\PYGdefault@tc##1{\textcolor[rgb]{0.73,0.13,0.13}{##1}}}
\expandafter\def\csname PYGdefault@tok@mb\endcsname{\def\PYGdefault@tc##1{\textcolor[rgb]{0.40,0.40,0.40}{##1}}}
\expandafter\def\csname PYGdefault@tok@mf\endcsname{\def\PYGdefault@tc##1{\textcolor[rgb]{0.40,0.40,0.40}{##1}}}
\expandafter\def\csname PYGdefault@tok@mh\endcsname{\def\PYGdefault@tc##1{\textcolor[rgb]{0.40,0.40,0.40}{##1}}}
\expandafter\def\csname PYGdefault@tok@mi\endcsname{\def\PYGdefault@tc##1{\textcolor[rgb]{0.40,0.40,0.40}{##1}}}
\expandafter\def\csname PYGdefault@tok@il\endcsname{\def\PYGdefault@tc##1{\textcolor[rgb]{0.40,0.40,0.40}{##1}}}
\expandafter\def\csname PYGdefault@tok@mo\endcsname{\def\PYGdefault@tc##1{\textcolor[rgb]{0.40,0.40,0.40}{##1}}}
\expandafter\def\csname PYGdefault@tok@ch\endcsname{\let\PYGdefault@it=\textit\def\PYGdefault@tc##1{\textcolor[rgb]{0.25,0.50,0.50}{##1}}}
\expandafter\def\csname PYGdefault@tok@cm\endcsname{\let\PYGdefault@it=\textit\def\PYGdefault@tc##1{\textcolor[rgb]{0.25,0.50,0.50}{##1}}}
\expandafter\def\csname PYGdefault@tok@cpf\endcsname{\let\PYGdefault@it=\textit\def\PYGdefault@tc##1{\textcolor[rgb]{0.25,0.50,0.50}{##1}}}
\expandafter\def\csname PYGdefault@tok@c1\endcsname{\let\PYGdefault@it=\textit\def\PYGdefault@tc##1{\textcolor[rgb]{0.25,0.50,0.50}{##1}}}
\expandafter\def\csname PYGdefault@tok@cs\endcsname{\let\PYGdefault@it=\textit\def\PYGdefault@tc##1{\textcolor[rgb]{0.25,0.50,0.50}{##1}}}

\makeatother

\makeatletter
\def\PYG@reset{\let\PYG@it=\relax \let\PYG@bf=\relax%
    \let\PYG@ul=\relax \let\PYG@tc=\relax%
    \let\PYG@bc=\relax \let\PYG@ff=\relax}
\def\PYG@tok#1{\csname PYG@tok@#1\endcsname}
\def\PYG@toks#1+{\ifx\relax#1\empty\else%
    \PYG@tok{#1}\expandafter\PYG@toks\fi}
\def\PYG@do#1{\PYG@bc{\PYG@tc{\PYG@ul{%
    \PYG@it{\PYG@bf{\PYG@ff{#1}}}}}}}
\def\PYG#1#2{\PYG@reset\PYG@toks#1+\relax+\PYG@do{#2}}

\expandafter\def\csname PYG@tok@w\endcsname{\def\PYG@tc##1{\textcolor[rgb]{0.73,0.73,0.73}{##1}}}
\expandafter\def\csname PYG@tok@c\endcsname{\let\PYG@it=\textit\def\PYG@tc##1{\textcolor[rgb]{0.25,0.50,0.50}{##1}}}
\expandafter\def\csname PYG@tok@cp\endcsname{\def\PYG@tc##1{\textcolor[rgb]{0.74,0.48,0.00}{##1}}}
\expandafter\def\csname PYG@tok@k\endcsname{\let\PYG@bf=\textbf\def\PYG@tc##1{\textcolor[rgb]{0.00,0.50,0.00}{##1}}}
\expandafter\def\csname PYG@tok@kp\endcsname{\def\PYG@tc##1{\textcolor[rgb]{0.00,0.50,0.00}{##1}}}
\expandafter\def\csname PYG@tok@kt\endcsname{\def\PYG@tc##1{\textcolor[rgb]{0.69,0.00,0.25}{##1}}}
\expandafter\def\csname PYG@tok@o\endcsname{\def\PYG@tc##1{\textcolor[rgb]{0.40,0.40,0.40}{##1}}}
\expandafter\def\csname PYG@tok@ow\endcsname{\let\PYG@bf=\textbf\def\PYG@tc##1{\textcolor[rgb]{0.67,0.13,1.00}{##1}}}
\expandafter\def\csname PYG@tok@nb\endcsname{\def\PYG@tc##1{\textcolor[rgb]{0.00,0.50,0.00}{##1}}}
\expandafter\def\csname PYG@tok@nf\endcsname{\def\PYG@tc##1{\textcolor[rgb]{0.00,0.00,1.00}{##1}}}
\expandafter\def\csname PYG@tok@nc\endcsname{\let\PYG@bf=\textbf\def\PYG@tc##1{\textcolor[rgb]{0.00,0.00,1.00}{##1}}}
\expandafter\def\csname PYG@tok@nn\endcsname{\let\PYG@bf=\textbf\def\PYG@tc##1{\textcolor[rgb]{0.00,0.00,1.00}{##1}}}
\expandafter\def\csname PYG@tok@ne\endcsname{\let\PYG@bf=\textbf\def\PYG@tc##1{\textcolor[rgb]{0.82,0.25,0.23}{##1}}}
\expandafter\def\csname PYG@tok@nv\endcsname{\def\PYG@tc##1{\textcolor[rgb]{0.10,0.09,0.49}{##1}}}
\expandafter\def\csname PYG@tok@no\endcsname{\def\PYG@tc##1{\textcolor[rgb]{0.53,0.00,0.00}{##1}}}
\expandafter\def\csname PYG@tok@nl\endcsname{\def\PYG@tc##1{\textcolor[rgb]{0.63,0.63,0.00}{##1}}}
\expandafter\def\csname PYG@tok@ni\endcsname{\let\PYG@bf=\textbf\def\PYG@tc##1{\textcolor[rgb]{0.60,0.60,0.60}{##1}}}
\expandafter\def\csname PYG@tok@na\endcsname{\def\PYG@tc##1{\textcolor[rgb]{0.49,0.56,0.16}{##1}}}
\expandafter\def\csname PYG@tok@nt\endcsname{\let\PYG@bf=\textbf\def\PYG@tc##1{\textcolor[rgb]{0.00,0.50,0.00}{##1}}}
\expandafter\def\csname PYG@tok@nd\endcsname{\def\PYG@tc##1{\textcolor[rgb]{0.67,0.13,1.00}{##1}}}
\expandafter\def\csname PYG@tok@s\endcsname{\def\PYG@tc##1{\textcolor[rgb]{0.73,0.13,0.13}{##1}}}
\expandafter\def\csname PYG@tok@sd\endcsname{\let\PYG@it=\textit\def\PYG@tc##1{\textcolor[rgb]{0.73,0.13,0.13}{##1}}}
\expandafter\def\csname PYG@tok@si\endcsname{\let\PYG@bf=\textbf\def\PYG@tc##1{\textcolor[rgb]{0.73,0.40,0.53}{##1}}}
\expandafter\def\csname PYG@tok@se\endcsname{\let\PYG@bf=\textbf\def\PYG@tc##1{\textcolor[rgb]{0.73,0.40,0.13}{##1}}}
\expandafter\def\csname PYG@tok@sr\endcsname{\def\PYG@tc##1{\textcolor[rgb]{0.73,0.40,0.53}{##1}}}
\expandafter\def\csname PYG@tok@ss\endcsname{\def\PYG@tc##1{\textcolor[rgb]{0.10,0.09,0.49}{##1}}}
\expandafter\def\csname PYG@tok@sx\endcsname{\def\PYG@tc##1{\textcolor[rgb]{0.00,0.50,0.00}{##1}}}
\expandafter\def\csname PYG@tok@m\endcsname{\def\PYG@tc##1{\textcolor[rgb]{0.40,0.40,0.40}{##1}}}
\expandafter\def\csname PYG@tok@gh\endcsname{\let\PYG@bf=\textbf\def\PYG@tc##1{\textcolor[rgb]{0.00,0.00,0.50}{##1}}}
\expandafter\def\csname PYG@tok@gu\endcsname{\let\PYG@bf=\textbf\def\PYG@tc##1{\textcolor[rgb]{0.50,0.00,0.50}{##1}}}
\expandafter\def\csname PYG@tok@gd\endcsname{\def\PYG@tc##1{\textcolor[rgb]{0.63,0.00,0.00}{##1}}}
\expandafter\def\csname PYG@tok@gi\endcsname{\def\PYG@tc##1{\textcolor[rgb]{0.00,0.63,0.00}{##1}}}
\expandafter\def\csname PYG@tok@gr\endcsname{\def\PYG@tc##1{\textcolor[rgb]{1.00,0.00,0.00}{##1}}}
\expandafter\def\csname PYG@tok@ge\endcsname{\let\PYG@it=\textit}
\expandafter\def\csname PYG@tok@gs\endcsname{\let\PYG@bf=\textbf}
\expandafter\def\csname PYG@tok@gp\endcsname{\let\PYG@bf=\textbf\def\PYG@tc##1{\textcolor[rgb]{0.00,0.00,0.50}{##1}}}
\expandafter\def\csname PYG@tok@go\endcsname{\def\PYG@tc##1{\textcolor[rgb]{0.53,0.53,0.53}{##1}}}
\expandafter\def\csname PYG@tok@gt\endcsname{\def\PYG@tc##1{\textcolor[rgb]{0.00,0.27,0.87}{##1}}}
\expandafter\def\csname PYG@tok@err\endcsname{\def\PYG@bc##1{\setlength{\fboxsep}{0pt}\fcolorbox[rgb]{1.00,0.00,0.00}{1,1,1}{\strut ##1}}}
\expandafter\def\csname PYG@tok@kc\endcsname{\let\PYG@bf=\textbf\def\PYG@tc##1{\textcolor[rgb]{0.00,0.50,0.00}{##1}}}
\expandafter\def\csname PYG@tok@kd\endcsname{\let\PYG@bf=\textbf\def\PYG@tc##1{\textcolor[rgb]{0.00,0.50,0.00}{##1}}}
\expandafter\def\csname PYG@tok@kn\endcsname{\let\PYG@bf=\textbf\def\PYG@tc##1{\textcolor[rgb]{0.00,0.50,0.00}{##1}}}
\expandafter\def\csname PYG@tok@kr\endcsname{\let\PYG@bf=\textbf\def\PYG@tc##1{\textcolor[rgb]{0.00,0.50,0.00}{##1}}}
\expandafter\def\csname PYG@tok@bp\endcsname{\def\PYG@tc##1{\textcolor[rgb]{0.00,0.50,0.00}{##1}}}
\expandafter\def\csname PYG@tok@fm\endcsname{\def\PYG@tc##1{\textcolor[rgb]{0.00,0.00,1.00}{##1}}}
\expandafter\def\csname PYG@tok@vc\endcsname{\def\PYG@tc##1{\textcolor[rgb]{0.10,0.09,0.49}{##1}}}
\expandafter\def\csname PYG@tok@vg\endcsname{\def\PYG@tc##1{\textcolor[rgb]{0.10,0.09,0.49}{##1}}}
\expandafter\def\csname PYG@tok@vi\endcsname{\def\PYG@tc##1{\textcolor[rgb]{0.10,0.09,0.49}{##1}}}
\expandafter\def\csname PYG@tok@vm\endcsname{\def\PYG@tc##1{\textcolor[rgb]{0.10,0.09,0.49}{##1}}}
\expandafter\def\csname PYG@tok@sa\endcsname{\def\PYG@tc##1{\textcolor[rgb]{0.73,0.13,0.13}{##1}}}
\expandafter\def\csname PYG@tok@sb\endcsname{\def\PYG@tc##1{\textcolor[rgb]{0.73,0.13,0.13}{##1}}}
\expandafter\def\csname PYG@tok@sc\endcsname{\def\PYG@tc##1{\textcolor[rgb]{0.73,0.13,0.13}{##1}}}
\expandafter\def\csname PYG@tok@dl\endcsname{\def\PYG@tc##1{\textcolor[rgb]{0.73,0.13,0.13}{##1}}}
\expandafter\def\csname PYG@tok@s2\endcsname{\def\PYG@tc##1{\textcolor[rgb]{0.73,0.13,0.13}{##1}}}
\expandafter\def\csname PYG@tok@sh\endcsname{\def\PYG@tc##1{\textcolor[rgb]{0.73,0.13,0.13}{##1}}}
\expandafter\def\csname PYG@tok@s1\endcsname{\def\PYG@tc##1{\textcolor[rgb]{0.73,0.13,0.13}{##1}}}
\expandafter\def\csname PYG@tok@mb\endcsname{\def\PYG@tc##1{\textcolor[rgb]{0.40,0.40,0.40}{##1}}}
\expandafter\def\csname PYG@tok@mf\endcsname{\def\PYG@tc##1{\textcolor[rgb]{0.40,0.40,0.40}{##1}}}
\expandafter\def\csname PYG@tok@mh\endcsname{\def\PYG@tc##1{\textcolor[rgb]{0.40,0.40,0.40}{##1}}}
\expandafter\def\csname PYG@tok@mi\endcsname{\def\PYG@tc##1{\textcolor[rgb]{0.40,0.40,0.40}{##1}}}
\expandafter\def\csname PYG@tok@il\endcsname{\def\PYG@tc##1{\textcolor[rgb]{0.40,0.40,0.40}{##1}}}
\expandafter\def\csname PYG@tok@mo\endcsname{\def\PYG@tc##1{\textcolor[rgb]{0.40,0.40,0.40}{##1}}}
\expandafter\def\csname PYG@tok@ch\endcsname{\let\PYG@it=\textit\def\PYG@tc##1{\textcolor[rgb]{0.25,0.50,0.50}{##1}}}
\expandafter\def\csname PYG@tok@cm\endcsname{\let\PYG@it=\textit\def\PYG@tc##1{\textcolor[rgb]{0.25,0.50,0.50}{##1}}}
\expandafter\def\csname PYG@tok@cpf\endcsname{\let\PYG@it=\textit\def\PYG@tc##1{\textcolor[rgb]{0.25,0.50,0.50}{##1}}}
\expandafter\def\csname PYG@tok@c1\endcsname{\let\PYG@it=\textit\def\PYG@tc##1{\textcolor[rgb]{0.25,0.50,0.50}{##1}}}
\expandafter\def\csname PYG@tok@cs\endcsname{\let\PYG@it=\textit\def\PYG@tc##1{\textcolor[rgb]{0.25,0.50,0.50}{##1}}}

\def\PYGZus{\char`\_}

\def\PYGZpc{\char`\%}

\makeatother

\usepackage{setspace}
\jmlrheading{1}{2020}{1-48}{4/00}{10/00}{meila00a}{Théophile Champion and Lancelot Da Costa and Howard Bowman and Marek Grze\'s}

\ShortHeadings{Branching Time Active Inference}{Champion et al.}
\firstpageno{1}

\begin{document}

\title{Branching Time Active Inference:\\
{\small the theory and its generality}}

\author{\name Théophile Champion \email tmac3@kent.ac.uk \\
       \addr University of Kent, School of Computing\\
       Canterbury CT2 7NZ, United Kingdom
       \AND
       \name Lancelot Da Costa \email l.da-costa@imperial.ac.uk \\
       \addr Imperial College London, Department of Mathematics\\
       London SW7 2AZ, United Kingdom\\
       Wellcome Centre for Human Neuroimaging, University College London\\
       London, WC1N 3AR, United Kingdom
       \AND
       \name Howard Bowman \email H.Bowman@kent.ac.uk \\
       \addr University of Birmingham, School of Psychology,\\
       Birmingham B15 2TT, United Kingdom\\
       University of Kent, School of Computing\\
       Canterbury CT2 7NZ, United Kingdom
       \AND
       \name Marek Grze\'s \email m.grzes@kent.ac.uk \\
       \addr University of Kent, School of Computing\\
       Canterbury CT2 7NZ, United Kingdom
       }
       
\editor{\textbf{TO BE FILLED}}

\maketitle

\begin{abstract}
Over the last 10 to 15 years, active inference has helped to explain various brain mechanisms from habit formation to dopaminergic discharge and even modelling curiosity. However, the current implementations suffer from an exponential (space and time) complexity class when computing the prior over all the possible policies up to the time-horizon. \citet{DeepAIwithMCMC} used Monte Carlo tree search to address this problem, leading to impressive results in two different tasks. In this paper, we present an alternative framework that aims to unify tree search and active inference by casting planning as a structure learning problem. Two tree search algorithms are then presented. The first propagates the expected free energy forward in time (i.e., towards the leaves), while the second propagates it backward (i.e., towards the root). Then, we demonstrate that forward and backward propagations are related to active inference and sophisticated inference, respectively, thereby clarifying the differences between those two planning strategies.
\end{abstract}


\section{Introduction}

Active inference is at this point a compelling explanatory approach in cognitive neuroscience, and significant analyses of biologically-realistic implementations in both neural and non-neural communication networks has been conducted. More specifically, active inference extends the free energy principle to generative models with actions \citep{FRISTON2016862,AI_TUTO,AI_VMP} and can be regarded as a form of planning as inference \citep{PAI}. This framework has successfully explained a wide range of neuro-cognitive phenomena, such as habit formation \citep{FRISTON2016862}, Bayesian surprise \citep{bayes_surprise}, curiosity \citep{curiosity}, and dopaminergic discharges \citep{dopamine}. It has also been applied to a variety of tasks, such as animal navigation \citep{DeepAIwithMCMC}, robotic control \citep{pezzato2020active,sancaktar2020endtoend}, the mountain car problem \citep{catal2020learning}, the game of DOOM \citep{CULLEN2018809} and the cart pole problem \citep{cart_pole}. Many of those applications require planning several steps into the future in order to be solved successfully. However, as explained in more depth in appendix H, an exhaustive search over all possible sequences of actions will quickly become intractable, i.e., the number of sequences to explore grows exponentially with the time horizon of planning. Figure \ref{fig:exp_class_illustration} illustrates this exponential growth. Exploring only a subset of this exponential number of possible sequences using a tree search therefore becomes a compelling and quite natural alternative.

\begin{figure}[H]
	\begin{center}
	\includegraphics[scale=0.8]{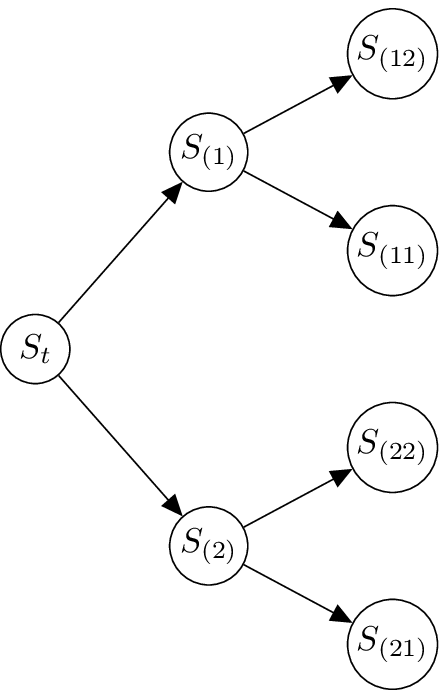}
	\end{center}
    \caption{Illustration of all possible policies up to two time steps in the future when $|U| = 2$. The state at the current time step is denoted by $S_t$. Additionally, each branch of the tree corresponds to a possible policy, and each node $S_I$ is indexed by a multi-index (e.g. $I=(12)$) representing the sequence of actions that led to this state. This should make it clear that for one time step in the future, there are $|U|$ possible policies, after two time steps there are $|U|$ times more policies, and so on until the time-horizon $T$ where there are a total of $|U|^T$ possible policies, i.e., the number of possible policies grows exponentially with the number of time steps for which the agent tries to plan.}
    \label{fig:exp_class_illustration}
\end{figure}

But what exactly is active inference? Imagine a basketball player at the top of the key (i.e., the area just below the net) ready to take a shot. Intuitively, active inference sees the world as a collection of external states such as the positions of the net, the player and the ball. The player (or agent) is equipped with sensors (such as the eyes) which allow for measurements of the external states. The player is also able to perform actions in the world such as to perform sudden eye movement or simply unfolding his (or her) arms and legs. Furthermore, it is believed that the agent stores an internal representation of the external states, that we shall refer to as the internal states. Importantly, the external and internal states are separated from each other by the Markov blanket \citep{MarkovBlanket}, i.e., the sensory information received and actions taken by the agent. In other words, the external states can only modify the internal states indirectly through the observations (also called sensory information) made by the agent, and the internal states can only modify the external states indirectly through the actions taken by the agent. 

More formally, active inference builds on a subfield of Bayesian statistics called variational inference \citep{VI_TUTO}, in which the true posterior distribution is approximated with a variational distribution. This method provides a way to balance the complexity and accuracy of the posterior distribution. The variational approach is only tractable because some statistical dependencies are ignored during the inference process, i.e., the variational distribution is generally assumed to fully factorise, leading to the well known mean-field approximation:
\begin{align}
Q(X) = \prod_{i} Q(X_i),
\end{align}
where $X$ is the set of all hidden variables of the model, $X_i$ represents the i-th hidden variable, $Q(X)$ is the variational distribution (see below) approximating the posterior $P(X|O)$ where $O$ is the available data, and $Q(X_i)$ is the $i$-th factor of the variational distribution. In 2005, \citet{VMP_TUTO} presented a message-based implementation of variational inference, which has naturally been called variational message passing. And more recently, \citet{AI_VMP} realised an active inference scheme using this variational message passing procedure. By combining the Forney factor graph formalism \citep{FFG_TUTO} with the method of \citet{VMP_TUTO}, it becomes possible to create modular implementations of active inference \citep{Simul_AI,DBLP:journals/ijar/CoxLV19} that allows users to define their own generative models without the burden of deriving update equations.

However, as just stated, there is a major bottleneck to scaling up the active inference framework: the number of action sequences grows exponentially with the time-horizon (see Appendix H for details). In the reinforcement learning literature, this explosion is frequently handled using Monte Carlo tree search (MCTS) \citep{Go,6145622,MuZero}. This approach has been applied to active inference in several papers \citep{DeepAIwithMCMC,LargePOMDP}. \citet{DeepAIwithMCMC} chose to modify the original criterion used during the node selection step in MCTS. This step returns the node that needs to be expanded, and the reinforcement learning community uses the upper confidence bound for trees (UCT) introduced by \citet{DBLP:conf/ecml/KocsisS06} as a selection criterion:
\begin{equation}\label{eq:uct0}
UCT_j = \bar{X}_j + 2 C_p \sqrt{\frac{2\ln n}{n_j}},
\end{equation}
where $n$ is the number of times the current (parent) node has been explored; $n_j$ stands for the number of times the j-th child node has been explored; $C_p > 0$ is the exploration constant and $\bar{X}_j$ is the average reward received by the j-th child, i.e., the sum of all rewards received by the current node and its descendants divided by $n_j$. The child node with the largest $UCT_j$ is selected. In their paper, \citet{DeepAIwithMCMC} replaced this selection criterion by:
\begin{equation}\label{eq:uct1}
U(s, a) = -\tilde{G}(s, a) + C_{\text{explore}} \,\, Q(a|s) \,\, \frac{1}{1 + N(s, a)}
\end{equation}
where $U(s, a)$ indicates the utility of selecting action $a$ in state $s$; $N(s, a)$ is the number of times that action $a$ was explored in state $s$; $C_{\text{explore}}$ is an exploration constant equivalent to $C_p$ in the UCT criterion; $Q(a|s)$ is a neural network modelling the posterior distribution over actions, which is trained by minimizing the variational free energy, and $\tilde{G}(s, a)$ is an estimator of the expected free energy (EFE). The EFE is computed from the following equation:
\begin{align}
G(\pi, \tau) = &- \mathbb{E}_{Q(\theta|\pi)Q(s_\tau|\theta,\pi)Q(o_\tau|s_\tau,\theta,\pi)}\Big[\ln P(o_\tau|\pi)\Big]\\
&+ \mathbb{E}_{Q(\theta|\pi)}\Big[\mathbb{E}_{Q(o_\tau|\theta,\pi)}H(s_\tau|o_\tau,\pi) - H(s_\tau|\pi)\Big]\\
&+ \mathbb{E}_{Q(\theta|\pi)Q(s_\tau|\theta,\pi)}H(o_\tau|s_\tau,\theta,\pi) - \mathbb{E}_{Q(s_\tau|\pi)}H(o_\tau|s_\tau,\pi),
\end{align}
where $H(x|y)$ is the entropy of $p(x|y)$. The computation of the EFE is performed by sampling from three distributions whose parameters are predicted by deep neural networks, i.e., the encoder network modelling $Q(s_\tau)$, the decoder network modelling $P(o_\tau|s_\tau)$ and the transition network modelling $P(s_\tau|s_{\tau-1},a_{\tau-1})$. Note that Equation \eqref{eq:uct0} was developed by \citet{DBLP:conf/ecml/KocsisS06} as a criterion for selecting nodes during planning, such that the selected node minimizes the agent's regret (c.f. Appendix G for additional details). Equation \eqref{eq:uct1} finds its origin in the Predictor Upper Confidence Bound (PUCB) algorithm introduced by \citet{PUCB}. The idea of the PUCB algorithm is to use contextual information to predict the node to select during planning. Equations \eqref{eq:uct0} and \eqref{eq:uct1} both aim to select the node that minimizes the agent's regret, and can therefore be used interchangeably. However, Equation \eqref{eq:uct1} requires contextual information and a model predicting the node to be selected. \citet{DeepAIwithMCMC} proposed to use the neural network modelling $Q(a|s)$ as a predictor. This has the advantage of making the predictor very flexible, since neural networks are known to be general function approximators, but neural networks are also expensive to train and lack interpretability.

To avoid the additional complexity brought by the predictor, this paper makes use of \eqref{eq:uct0}, which arises from the multi-armed bandit literature \citep{Auer2002}. The idea is to minimise the agent's regret to handle the trade-off between exploration and exploitation at the tree-level in an optimal manner.

A major novelty of our paper is to think about tree search as a dynamical expansion of the generative model, where the past and present is modelled as a partially observable Markov decision process \citep{POMDP_THESIS} and the future is modelled by a tree-like generative model. Importantly, our agent treats future states and observations as latent variables over which posterior beliefs are computed, and those beliefs encode the uncertainty of our agent over future states. In contrast, \citet{DeepAIwithMCMC} are using a maximum a posteriori (MAP) estimate of the future hidden states, while performing MCTS. Lastly, the posterior beliefs held by our agent are computed using variational message passing as presented in \citep{AI_VMP}. In comparison, \citet{DeepAIwithMCMC} perform amortized inference using an encoder network that predicts the mean and variance of the posterior distribution over latent states. Then, (during planning) a MAP estimate is used as input for the neural network modelling the temporal transition. All those neural networks are trained using gradient descent on the variational free energy.

Overall, the key contribution of our paper is to use MCTS to expand or grow the probabilistic graphical model, treat future states and observations as latent variables, and do inference using variational message passing. Indeed, the definition in \citep{AI_VMP} of a general message passing procedure for performing active inference makes it possible to construct graphical active inference models in a modular fashion. In turn, this makes it possible to incrementally expand an active inference model as required of our MCTS procedure. It is this message passing procedure that makes our approach possible. To our knowledge, an approach of this kind has never been studied before.

In the following, we first provide the requisite background concerning Forney factor graphs, variational message passing, active inference, and Monte Carlo tree search in Sections \ref{sec:ffg}, \ref{sec:vmp}, \ref{sec:ai}, and \ref{sec:MCTS}, respectively. Next, Section \ref{sec:ai_ts} introduces our method that frames planning using a tree as a form of Bayesian model extension. Using terminology from concurrency theory \citep{concurrency_howard}, we call our new formalism \textit{Branching Time Active Inference} (BTAI). In this domain, models of systems based upon sequences of actions (the format of policies) are described as \textit{linear time}, while models based upon tree and even graph structures are called \textit{branching time} \citep{concurrency_glabbeek,concurrency_glabbeek_2,concurrency_howard}. Importantly, BTAI does not consider the generative model and the tree as two different objects, instead, BTAI merges those two objects together into a generative model that can be dynamically expanded. For a detailed analysis of the properties of BTAI, the reader is referred to our companion paper \citep{AITS_PRACTICE}, which provides an empirical demonstration of the benefits of BTAI over standard active inference (AcI) in the context of a graph navigation task. This companion paper also supplies a theoretical comparison of BTAI and standard AcI based upon a complexity class analysis. Briefly, standard AcI has a space complexity class of $O(|\pi| \times T \times |S|)$, where $|\pi| = |U|^T$ is the number of possible policies, $|U|$ is the number of available actions, $T$ is the time horizon of planning, and $|S|$ is the number of values that the hidden state can take. In contrast, the space complexity class of BTAI is $O([K + t] \times |S|)$, where $t$ is the current (i.e. present) time point, and $K$ is the number of expansions of the tree performed during planning. Importantly, even complex applications such as the game of Go can be solved by expanding only a small number of nodes \citep{Go,MuZero}. Section \ref{sec:ai_ts} is followed by Section \ref{sec:aits_ai_sai} that explains the connection between our method and the planning strategies used in both active inference and sophisticated inference \citep{Sophisticated_INF}. Finally, Section \ref{sec:conclusion} concludes this paper and provides ideas for future research.

\section{Forney Factor Graphs}\label{sec:ffg}

A Forney factor graph \citep{FFG_TUTO} uses three kinds of nodes. The nodes representing hidden and observed variables are depicted by white and gray circles, respectively. And the distribution's factors are represented using white squares, which are linked to variable nodes by arrows or lines. Arrows are used to connect factors to their target variable, while lines link factors to their predictors. Figure \ref{fig:type-elements-ffg} shows an example of a Forney factor graph corresponding to the following generative model:
\begin{align} \label{eq:general-generative-model}
P(O,S) = {\color{red!90!black}P_{O}}(O|S){\color{blue!50!black}P_{S}}(S).
\end{align}

Generally, factor graphs only describe the model's structure such as the variables and their dependencies, but do not specify the definition of individual factors. For example, the definitions of ${\color{red!90!black}P_{O}}$ and ${\color{blue!50!black}P_{S}}$ are not given by Figure \ref{fig:type-elements-ffg}, and additional information is required to remove the ambiguity, e.g., ${\color{blue!50!black}P_{S}}(S) = \mathcal{N}(S;\mu, \sigma)$ clarifies that ${\color{blue!50!black}P_{S}}$ is a Gaussian distribution.

\begin{figure}[H]
	\begin{center}
	\includegraphics[scale=1]{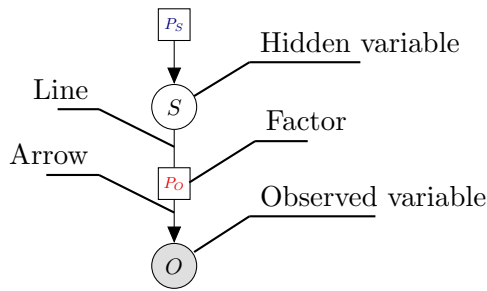}
	\end{center}
	\caption[Forney Factor Graph]{This figure illustrates the Forney factor graph corresponding to the following generative model: $P(O,S) = {\color{red!90!black}P_{O}}(O|S){\color{blue!50!black}P_{S}}(S)$. The hidden state is represented by a white circle with the variable's name at the center, and the observed variable is depicted similarly but with a gray background. The factors of the generative model are represented by squares with a white background and the factor's name at the center. Finally, arrows connect the factors to their target variable and lines link each factor to its predictor variables.}
    \label{fig:type-elements-ffg}
\end{figure}

\section{Variational Message Passing}\label{sec:vmp}

We now build on Forney factor graphs and provide an overview of the method of \citet{VMP_TUTO}. For more details, see \citet{AI_VMP}, which provided a complete derivation of the equations presented below from Bayes' theorem.

\subsection{Winn and Bishop method}

Variational message passing as developed by \citet{VMP_TUTO} is an approach for inference based upon the mean-field approximation, which assumes that the posterior fully factorises, i.e.
\begin{align}
Q(X) = \prod_{i} Q(X_i),
\end{align}
where $X$ is the set of all hidden variables of the model and $X_i$ represents the $i$-th hidden variable. In this section, we focus on the intuition behind the method, starting with the update equation of an arbitrary hidden state $x_k$:
\begin{align}\label{eq:5}
\ln Q_k^*(x_k) = \langle \ln P(x_k|{\color{purple}\text{pa}_k}) \rangle_{\sim Q_k} + \sum_{c_j \in {\color{red}\text{ch}_k}} \langle \ln P(c_j|x_k, {\color{violet}\text{cp}_{kj}}) \rangle_{\sim Q_k} + C
\end{align}
where $C$ is a normalizing constant, and $\langle \cdot \rangle_{\sim Q_k}$ is the expectation over all factors but $Q_k(x_k)$. \eqref{eq:5} tells us that the optimal posterior of any hidden states $x_k$ only depends on its Markov blanket, i.e., $x_k$'s parents ${\color{purple}\text{pa}_k}$, children ${\color{red}\text{ch}_k}$ and co-parents ${\color{violet}\text{cp}_{kj}}$. To make \eqref{eq:5} more specific, we assume that each random variable of the model is conjugate to its parents (i.e., the posterior has the same functional form as the prior) and is distributed according to a distribution in the exponential family, i.e., 
\begin{align}\label{eq:prior_VMP}
\ln P(x_k|\text{pa}_k) = \mu_k(\text{pa}_k) \cdot u_k(x_k) + h_k(x_k) + z_k(\text{pa}_k)
\end{align}
where $\mu_k(\text{pa}_k)$, $u_k(x_k)$, $h_k(x_k)$ and $z_k(\text{pa}_k)$ are the parameters, the sufficient statistics, the underlying measure and the log partition, respectively. Under those two assumptions, \eqref{eq:5} can be re-written as:
\begin{align}\label{eq:posterior_VMP}
Q_k^*(x_k) &= \exp \Bigg\{ {\color{blue}\mu_k^*} \cdot {\color{orange}u_k}(x_k) + {\color{red}h_k}(x_k) + \text{Const} \Bigg\}
\end{align}
\begin{align}\label{eq:6}
{\color{blue}\mu_k^*} = {\color{purple}\tilde{\mu}_{k}}(\{{\color{gray}\langle u_{i}(i) \rangle_{Q_{i}}} \}_{i \in \text{pa}_k} ) + \sum_{c_j \in \text{ch}_k} {\color{violet}\tilde{\mu}_{j \rightarrow k}}({\color{gray}\langle u_j(c_j) \rangle_{Q_j}}, \{{\color{gray}\langle u_{l}(l)\rangle_{Q_l}} \}_{l \in \text{cp}_{kj}} )
\end{align}
where ${\color{purple}\tilde{\mu}_{k}}$ is a re-parameterization of ${\color{purple}\mu_k}(pa_k)$ in terms of the expectation of the sufficient statistics of the parents of $x_k$, and similarly ${\color{violet}\tilde{\mu}_{j \rightarrow k}}$ is a re-parameterization of ${\color{violet}\mu_{j \rightarrow k}}$. Importantly, ${\color{orange}u_k}(x_k)$ and ${\color{red}h_k}(x_k)$ in the optimal posterior \eqref{eq:posterior_VMP} are the same as in the prior \eqref{eq:prior_VMP}, and only the parameters have changed according to \eqref{eq:6}.

\begin{figure}[H]
	\begin{center}
	\includegraphics[scale=1]{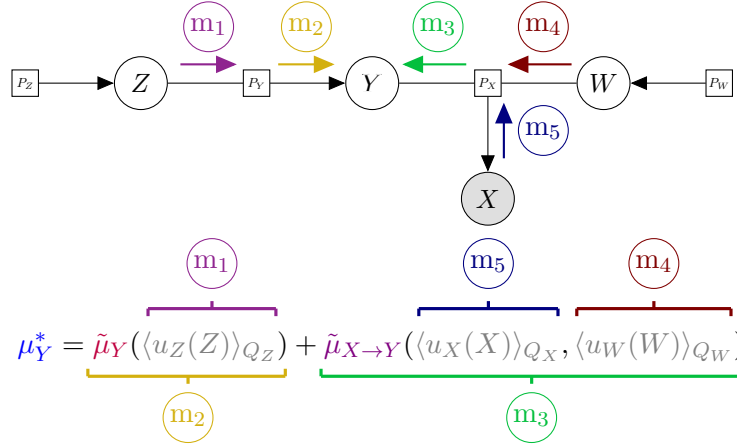}
	\end{center}
  \caption{This figure illustrates the computation of the optimal posterior parameters as a message passing procedure, which requires the transmission of messages from the parent (${\color{Yellow}\text{m}_2}$) and child (${\color{Green}\text{m}_3}$) factors. Additionally, the message from the child factor (${\color{Green}\text{m}_3}$) requires the computation of messages from the co-parent (${\color{Red}\text{m}_4}$) and child (${\color{Blue}\text{m}_5}$) variables. Also, the message from the parent factor (${\color{Yellow}\text{m}_2}$) requires the computation of a message (${\color{Violet}\text{m}_1}$) from the parent variable.}
   \label{fig:messages-ffg}
\end{figure}

To understand the intuition behind (\ref{eq:6}), let us suppose that we are given the Forney factor graph illustrated in Figure \ref{fig:messages-ffg} and we wish to compute the posterior of $Y$. Then, the only parent of $Y$ is $Z$, the only child of $Y$ is $X$ and the only co-parent of $Y$ with respect to $X$ is $W$. Therefore, applying \eqref{eq:6} to our example leads to the equation presented in Figure \ref{fig:messages-ffg} whose components can be interpreted as messages. Indeed, each variable (i.e., $X$, $Z$ and $W$) sends the expectation of its sufficient statistics (i.e., a message) to the square node in the direction of Y (i.e., either $P_X$ or $P_Y$). Those messages are then combined using a function (i.e., either $\tilde{\mu}_{Y}$ or $\tilde{\mu}_{X \rightarrow Y}$) whose output (i.e., another set of messages) are summed to obtain the optimal parameters $\mu_Y^*$. The computation of the optimal parameters (\ref{eq:6}) can then be understood as a message passing procedure. Also, we provide in Appendix C a concrete instance of the approach presented above.

\section{Active Inference}\label{sec:ai}

This section provides a quick overview of the active inference framework, and Appendix H presents a description of the exponential complexity class that it exhibits. The reader is referred to Appendix F for any notations that might not be explained here. For a more detailed treatment of the active inference framework, we refer the reader to \citep{AI_VMP,AI_TUTO,TUTO_AI_RYAN}.

\subsection{Generative model} \label{ssec:GM}

As illustrated in Figure \ref{fig:full_GM}, the classic generative model represents the world as a sequence of hidden states generating observations through the matrix $\bm{A}$. The prior over the initial states is defined by the vector $\bm{D}$ and the transition between time steps is encoded by a 3-tensor $\bm{B}$, i.e., one matrix per action. Importantly, the random variable $\pi$ represents all possible policies up to a given time horizon $T$ and each policy is defined as a sequence of actions, i.e., $\{U_t, ..., U_{T - 1}\}$ where $U_\tau \in \{1, ..., |U|\} \,\, \forall \tau \in \{t, ..., T - 1\}$. The prior over the policies is then set such that policies with high probability minimise the EFE, which is defined as follows \citep{Parr2019}:
\begin{align}\label{eq:1}
\bm{G}(\pi) \approx \sum_{\tau=t + 1}^T \Bigg[ \underbrace{D_{\mathrm{KL}}[\overbrace{Q(O_\tau|\pi)}^{\text{expected outcomes}}||\overbrace{P(O_\tau)}^{\text{prior preferences}}]}_{\text{risk}}\,\, +\,\, \underbrace{\mathbb{E}_{Q(S_\tau|\pi)}[\text{H}[P(O_\tau | S_\tau)]]}_{\text{ambiguity}}\Bigg]
\end{align}
where $Q(O_{\tau}|\pi) \delequal \sum_{S_\tau} P(O_\tau|S_\tau)Q(S_\tau|\pi)$, $\text{H}[\cdot]$ is the Shannon entropy, $\bm{G}$ is a vector containing as many elements as the number of policies, and the i-th element of $\bm{G}$ represents the cost of the i-th policy. The prior preferences over observations $P(O_\tau)$ represent the (categorical) distribution that the agent wants its observations to be sampled from and is traditionally encoded by the vector $\bm{C}$. Note that this generalises the concept of reward from reinforcement learning. Indeed, maximising reward can be reformulated as sampling observations from a Dirac delta distribution over reward maximising states \citep{dacosta2020relationship}.

\begin{figure}[H]
	\begin{center}
	\includegraphics[scale=0.9]{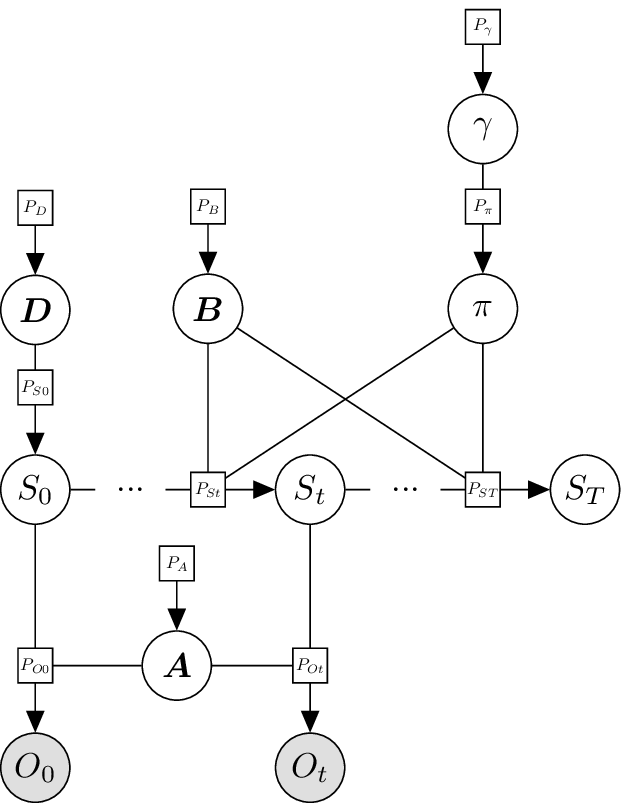}
	\end{center}
    \caption{This figure illustrates the Forney factor graph of the entire generative model presented by \citet{FRISTON2016862}. The probability of the initial states is defined by the vector $\bm{D}$, and the matrix $\bm{A}$ defines the probability of the observations given the hidden states. The $\bm{B}$ matrices define the transition between any successive pair of hidden states. This transition depends on the action performed by the agent, i.e., on the policy $\pi$. Furthermore, the prior over the policies has been chosen such that policies minimizing expected free energy are more probable. Finally, the precision parameter $\gamma$ (which modulates the confidence over which policies to pursue) is distributed according to a gamma distribution.}
    \label{fig:full_GM}
\end{figure}

Lastly, the precision parameter $\gamma$ has been associated to neuromodulators such as dopamine \citep{dopamine,AI_DISCRET} and can be understood as modulating the confidence over the information afforded by the expected free energy---e.g., smaller values of $\gamma$ lead to more stochastic decision-making. Finally, the framework allows $\bm{A}$, $\bm{B}$ and $\bm{D}$ to be learned by introducing Dirichlet distributions over the columns of these tensors such that the posterior parameters of $\bm{A}$, $\bm{B}$ and $\bm{D}$ can be reused in a new trial, as parameters of the prior, giving an empirical prior. Finally, the classic generative model is defined as follows:
\begin{align}
P(O_{0:t}, S_{0:T}, \pi, \bm{A}, \bm{B}, \bm{D}, \gamma)\,\, =\,\,\,\, &P(\pi|\gamma) P(\gamma) P(\bm{A}) P(\bm{B}) P(S_0|\bm{D}) P(\bm{D})\nonumber\\
&\prod^{t}_{\tau = 0} P(O_\tau|S_\tau,\bm{A}) \prod^{T}_{\tau = 1} P(S_\tau|S_{\tau-1},\pi_{\tau-1},\bm{B})\label{eq:2}
\end{align}
\vspace{-2cm}\\
\begin{align*}
&P(\pi|\gamma) = \sigma(-\gamma \bm{G}) & & P(\gamma) = \Gamma (1, \bm{\beta})\\
&P(\bm{A}) = \text{Dir}(\bm{a}) & &  P(\bm{B}) = \text{Dir}(\bm{b})\\
&P(S_0|\bm{D}) = \text{Cat}(\bm{D}) & & P(\bm{D}) = \text{Dir}(\bm{d})\\
&P(O_\tau|S_\tau,\bm{A}) = \text{Cat}(\bm{A}) & & P(S_{\tau}|S_{\tau - 1},\pi_{\tau-1},\bm{B}) = \text{Cat}(\bm{B}),
\end{align*}
where $\bm{G}$ is a vector of size $|\pi|$ whose $i$-th element corresponds to the expected free energy of the i-th policy, $\sigma(\bigcdot)$ is the softmax function, $\Gamma(\bigcdot)$, $\text{Cat}(\bigcdot)$ and $\text{Dir}(\bigcdot)$ stand for a gamma, categorical and Dirichlet distribution, respectively, $\pi_{\tau-1} \in \{1, ..., |U|\}$ is the action prescribed by policy $\pi$ at time $\tau-1$, $O_{0:t}$ is the set of (random variables representing) observations between time step $0$ and $t$, and $S_{0:T}$ is the set of (random variables representing) hidden states between time step $0$ and $T$.

\subsection{Variational Distribution}\label{section:VD}

The most widely used variational distribution \citep{AI_TUTO,FRISTON2016862} is not fully factorized, i.e., the posterior models the influence of the policy on the hidden states, leading to the following factorization: 
\vspace{-0.2cm}
\begin{align}\label{vd}
Q(S_{0:T}, \pi, \bm{A}, \bm{B}, \bm{D}, \gamma) = Q(\pi)Q(\bm{A})Q(\bm{B})Q(\bm{D})Q(\gamma) \prod_{\tau=0}^{T} Q(S_\tau|\pi)
\end{align}
\vspace{-2cm}\\
\begin{align*}
&Q(S_\tau|\pi) = \text{Cat}(\bm{\hat{D}}_\tau) && Q(\pi) = \text{Cat}(\bm{\hat{\pi}})\\
&Q(\gamma) = \Gamma(1,\bm{\hat{\beta}}) && Q(\bm{D}) = \text{Dir}(\bm{\hat{d}})\\
&Q(\bm{A}) = \text{Dir}(\bm{\hat{a}}) && Q(\bm{B}) = \text{Dir}(\bm{\hat{b}})
\end{align*}
where all variables with a hat correspond to posterior parameters. Notice that the distributions over $\bm{A}$, $\bm{B}$ and $\bm{D}$ remain Dirichlet distributions, and the distributions over $\gamma$ and $S_\tau$ remain a gamma and a categorical distribution, respectively. Only the distribution over $\pi$ changes from a Boltzmann to a categorical distribution but both are discrete distributions.

\begin{remark}
By definition the generative model $P(O_{0:t}, S_{0:T}, \pi, \bm{A}, \bm{B}, \bm{D}, \gamma)$ is a joint probability distribution over both the observed ($O_{0:t}$) and latent ($S_{0:T}, \pi, \bm{A}, \bm{B}, \bm{D}, \gamma$) variables. However, the goal of the variational distribution
$Q(S_{0:T}, \pi, \bm{A}, \bm{B}, \bm{D}, \gamma)$ is to approximate the true posterior $P(S_{0:T}, \pi, \bm{A}, \bm{B}, \bm{D}, \gamma | O_{0:t})$, which is a distribution over the latent variables only. Thus, the approximate posterior $Q(S_{0:T}, \pi, \bm{A}, \bm{B}, \bm{D}, \gamma)$ is also a distribution over the latent variables only, and does not contain the observed variables.
\end{remark}

\subsection{Variational Free Energy}

By definition, the variational free energy (VFE) is the Kullback-Leibler divergence between the variational distribution and the generative model, i.e.
\begin{align}
\bm{F} &= \mathbb{E}_{Q}[\ln Q(S_{0:T}, \pi, \bm{A}, \bm{B}, \bm{D}, \gamma) - \ln P(O_{0:t},S_{0:T}, \pi, \bm{A}, \bm{B}, \bm{D}, \gamma)]\\
&= \kl{Q(x)}{P(x|o)} \,\,\, - \ln P(o)\label{eq:3}\\
&= \underbrace{\kl{Q(x)}{P(x)}}_{\text{complexity}} - \underbrace{\mathbb{E}_{Q(x)}[\ln P(o|x)]}_{\text{accuracy}} \label{eq:4}
\end{align}
where $x = \{S_{0:T}, \pi, \bm{A}, \bm{B}, \bm{D}, \gamma\}$ refers to the model's hidden variables, and $o = \{O_{0:t}\}$ refers to the sequence of observations made by the agent. \eqref{eq:3} shows that minimising free energy involves moving the variational distribution $Q(x)$ closer to the true posterior $P(x|o)$ in the sense of KL divergence, and that the variational free energy is an upper bound on the negative log evidence. \eqref{eq:4} shows the trade-off between complexity and accuracy, where the complexity penalises the divergence of the posterior $Q(x)$ from the prior $P(x)$ and the accuracy scores how likely the observations are given the generative model and current belief of the hidden states.

To fit the variational distribution as closely as possible to the true posterior, the VFE is minimized w.r.t each variational factor, e.g., $Q(\bm{D})$ and $Q(\bm{A})$. The minimization process can be solved by iterating the update equations of each factor until convergence of the VFE. More details and intuition about those updates are given by \citet{AI_VMP}.

\subsection{Action selection}

In active inference, the simplest strategy to select actions is to compute the evidence for all policies under consideration and then choose the most likely action according to these policies. Mathematically, this amounts to a Bayesian model average by executing the action with the highest posterior evidence:
\begin{align}
u_t^* = \argmax_{u} \sum_{m = 1}^{|\pi|} [u = \pi^m_t] Q(\pi = m)
\end{align}
where $|\pi|$ is the number of policies, $\pi^m_t$ is the action predicted at the current time step by the $m-th$ policy, and $[u = \pi^m_t]$ is an indicator function that equals one if $u = \pi^m_t$ and zero otherwise.

\section{Monte Carlo Tree Search} \label{sec:MCTS}

By now, the reader should be familiar with the framework of active inference and how variational message passing combined with the Forney factor graph formalism can be used to compute posterior beliefs. We now turn to the last piece of background required to present the method proposed in this paper: Monte Carlo tree search (MCTS), which is based on the multi-armed bandit literature (c.f. Appendix G for details).

\subsection{A four step process}

Monte Carlo tree search has been widely used in the reinforcement learning literature as it enables agents to plan efficiently when the evaluation of every possible action sequence is computationally prohibitive \citep{Go,6145622,MuZero,DeepAIwithMCMC}. This algorithm essentially builds a tree in which each node corresponds to a future state and each edge represents the action that led to that state. Initially, the tree is only composed of a root node corresponding to the current state. From here, MCTS is a four step process. First, a node is selected according to a criterion such as the upper confidence bound for trees (UCT):
\begin{align}
UCT_j = \bar{X}_j + 2 C_p \sqrt{\frac{2\ln n}{n_j}},
\end{align}
where $n$ is the number of times the current (parent) node has been explored, $n_j$ stands for the number of times the j-th child node has been explored, $C_p > 0$ is the exploration constant and $\bar{X}_j$ is the average reward received by the j-th child. Note, if the rewards are in $[0,1]$, then $C_p = \frac{1}{\sqrt{2}}$ is known to satisfy the Hoeffding inequality \citep{6145622} and the UCT criterion reduces to:
\begin{align}
UCT_j = \bar{X}_j + 2 \sqrt{\frac{\ln n}{n_j}}.
\end{align}
Importantly, the UCT aims to explore highly rewarding paths (exploitation in first term), while also visiting rarely explored regions (exploration in second term). 

As shown in Figure \ref{fig:MCTS}, this criterion is first used at the root level leading to the selection of a node from the root's children. Then, it is used at the level of the root's children, and so on until a leaf node is reached. As explained by \citet{10.1007/11871842_29}, $UCT$ is a direct application of the $UCB1$ criterion to trees, where at each level, the allocation strategy must pick a node that is expected to lead to the highest reward, and ``picking the $i$-th node" can be seen as the $i$-th action of a multi-armed bandit problem. Once a leaf node has been selected, an expansion step is performed by sampling an action from a distribution and adding the node corresponding to this action as a child of the leaf node, i.e., the leaf node is expanded.

The third step consists of performing virtual rollouts into the future to estimate the average future reward obtained from the state corresponding to the newly expanded node. Finally, during the back-propagation step, the average reward obtained from the newly expanded state is used to re-evaluate the average quality of all its ancestors, and the visit counts of all nodes (in the branch explored) are increased. Iterating this four-step process until the time budget has been spent gives a fairly good estimate of the best action to perform next. Figure \ref{fig:MCTS} summarises the MCTS procedure. In the next section, we present our approach and show how MCTS can be fused to active inference by performing a dynamical expansion of the generative model.

\begin{figure}[H]
	\begin{center}
	\includegraphics[scale=1]{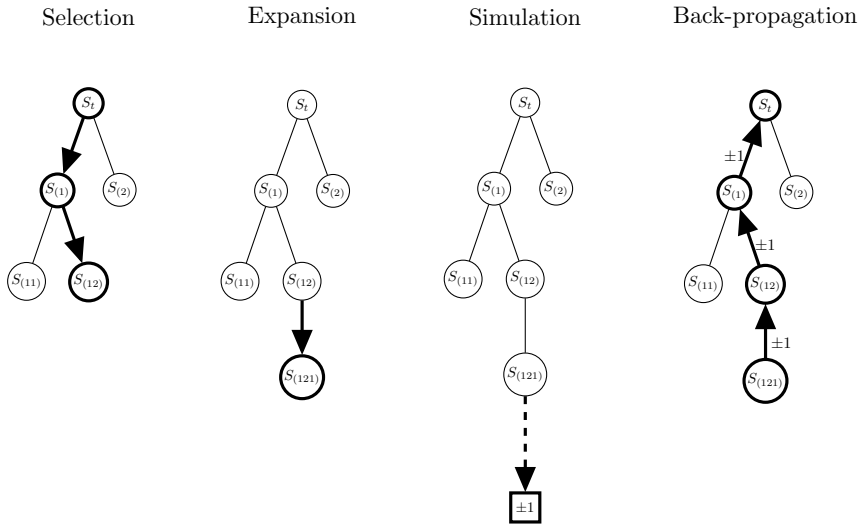}
	\end{center}
\vspace{-0.25cm}
    \caption{
This figure illustrates the MCTS algorithm as a four step process. First, we start at the node representing the current state $S_t$ and select a node based on the UCT criterion until a leaf node is reached. Second, the tree is expanded to a new node by taking a virtual action from the selected node. Third, the value of this action is estimated by simulating the expected reward following that action. In the simplest version of MCTS, simulations are run until a terminal state is reached, e.g., until the game ends in Go or Chess. Fourth, the expected value is back-propagated to the new node and all of its ancestor nodes. The multi-indices in curly brackets denote action sequences taken from the root node, indicating the current state of the environment.}
    \label{fig:MCTS}
\end{figure}

\section{Branching Time Active Inference (BTAI)} \label{sec:ai_ts}

In this section, we present a novel active inference agent that frames planning using a tree as a form of Bayesian model extension. Using terminology from concurrency theory \citep{concurrency_howard}, we call our new formalism, \textit{Branching Time Active Inference} (BTAI). In this domain, models of systems based upon sequences of actions (the format of policies) are described as \textit{linear time}, while models based upon tree and even graph structures are called \textit{branching time} \citep{concurrency_glabbeek,concurrency_glabbeek_2,concurrency_howard}. Importantly, we do not consider the generative model and the tree as two different objects. Instead, we merge those two objects together into a generative model that can be dynamically expanded.

Figure \ref{fig:AITS} illustrates an example of such a model, where for the sake of simplicity, we assume that the matrices $\bm{A}$, $\bm{B}$ and $\bm{D}$ are given to the agent. Furthermore, the random variable representing the policies has been replaced by random variables representing actions and the precision parameter $\gamma$ has been removed, which is a common design choice \citep{DeepAIwithMCMC}. Additionally, we follow \citet{Parr2019} by viewing future observations as latent random variables. Finally, note that the transition between two consecutive hidden states in the future ($S_\IdMLast{I}$ and $S_I$ where $I$ is a multi-index) will only depend on the matrix $\bm{\bar{B}}_I = \bm{\bar{B}}(\bigcdot, \bigcdot, I_\text{last})$, i.e., the matrix corresponding to action $I_\text{last}$ that led to the transition from $S_\IdMLast{I}$ to $S_I$. The reader is referred to Table \ref{tab:message_passing_notation} for the definition of $\bm{\bar{B}}$ and more details about multi-indices can be found in Appendix F.

\subsection{Prior, Posterior and Target distributions}

Since the generative model is fairly different from the standard model, we state here its formal definition:
\begin{align}
P(O_{0:t},&S_{0:t},U_{0:t-1},O_{\mathbb{I}_t},S_{\mathbb{I}_t},\bm{A},\bm{B},\bm{D},\bm{\Theta}_{0:t-1}) = P(S_0|\bm{D}) P(\bm{A}) P(\bm{B}) P(\bm{D}) \prod_{\tau = 0}^t P(O_\tau|S_\tau,\bm{A})\nonumber\\
& \prod_{\tau = 0}^{t - 1} P(U_\tau|\bm{\Theta}_\tau)P(\bm{\Theta}_\tau) \prod_{\tau = 1}^t P(S_\tau|S_{\tau - 1}, U_{\tau - 1},\bm{B})\prod_{I \in \mathbb{I}_t} P(O_I|S_I)P(S_I|S_\IdMLast{I})
\end{align}
where $\mathbb{I}_t$ is the set of all non-empty multi-indices already expanded by the tree search from the current state $S_t$, the second product ($\tau$ from $0$ to $t-1$) models the uncertainty over action, reflecting the focus on actions rather than policies, and $S_\IdMLast{I}$ is the parent state of $S_I$. Intuitively, the product over all $I \in \mathbb{I}_t$ models the future, while the rest of the above equation models the past and present. Additionally, we need to define the individual factors:
\begin{align*}
&P(S_0|\bm{D}) = \text{Cat}(\bm{D})& &P(U_\tau|\bm{\Theta}_\tau) = \text{Cat}(\bm{\Theta}_\tau) \\
&P(O_\tau|S_\tau,\bm{A}) = \text{Cat}(\bm{A})& &P(O_I|S_I) = \text{Cat}(\bm{\bar{A}}) \\
&P(S_\tau|S_{\tau - 1}, U_{\tau - 1},\bm{B}) = \text{Cat}(\bm{B})& &P(S_I|S_\IdMLast{I}) = \text{Cat}(\bm{\bar{B}}_I)\\
&P(\bm{D}) = \text{Dir}(\bm{d}) & & P(\bm{\Theta}_\tau) = \text{Dir}(\bm{\theta}_\tau)\\
&P(\bm{A}) = \text{Dir}(\bm{a}) & & P(\bm{B}) = \text{Dir}(\bm{b})
\end{align*}
where $\bm{\bar{A}}$ and $\bm{\bar{B}}$ are defined in Table \ref{tab:message_passing_notation}, $\bm{\bar{B}}_I = \bm{\bar{B}}(\bigcdot,\bigcdot,I_\text{last})$ is the matrix corresponding to $I_\text{last}$ and $I_\text{last}$ is the last index of the multi-index $I$, i.e., the last action that led to $S_I$. Importantly, $\bm{\bar{A}}$ and $\bm{\bar{B}}$ should not be confused with $\bm{\mathring{A}}$ and $\bm{\mathring{B}}$, $\bm{\bar{A}}$ is the expectation of $\bm{A}$ w.r.t. $Q(\bm{A})$, while $\bm{\mathring{A}}$ is the expectation of the logarithm of $\bm{A}$ w.r.t. $Q(\bm{A})$.

We now turn to the definition of the variational posterior. Under the mean-field approximation:
\vspace{-0.5cm}
\begin{equation}
\label{eq: mean field approx posterior}
\begin{split}
   Q(S_{0:t},U_{0:t-1},&O_{\mathbb{I}_t},S_{\mathbb{I}_t},\bm{A},\bm{B},\bm{D},\bm{\Theta}_{0:t-1}) = \\
&Q(\bm{A}) Q(\bm{B}) Q(\bm{D}) \prod_{\tau = 0}^{t - 1} Q(U_\tau) Q(\bm{\Theta}_\tau) \prod_{\tau = 0}^t Q(S_\tau) \prod_{I \in \mathbb{I}_t} Q(O_I)Q(S_I) 
\end{split}
\end{equation}
where the individual factors are defined as:
\begin{align*}
&Q(S_\tau) = \text{Cat}(\bm{\hat{D}}_\tau)& &Q(U_\tau) = \text{Cat}(\bm{\hat{\Theta}}_\tau) \\
&Q(O_I) = \text{Cat}(\bm{\hat{E}}_I)& &Q(S_I) = \text{Cat}(\bm{\hat{D}}_I)\\
&Q(\bm{D}) = \text{Dir}(\bm{\hat{d}}) & & Q(\bm{\Theta}_\tau) = \text{Dir}(\bm{\hat{\theta}}_\tau)\\
&Q(\bm{A}) = \text{Dir}(\bm{\hat{a}}) & & Q(\bm{B}) = \text{Dir}(\bm{\hat{b}}),
\end{align*}
where $\bm{\hat{D}}_\tau$, $\bm{\hat{\Theta}}_\tau$, $\bm{\hat{E}}_I$, $\bm{\hat{D}}_I$, $\bm{\hat{d}}$, $\bm{\hat{\theta}}_\tau$, $\bm{\hat{a}}$ and $\bm{\hat{b}}$ are the parameters of the factors $Q(S_\tau)$, $Q(U_\tau)$, $Q(O_I)$, $Q(S_I)$, $Q(\bm{D})$, $Q(\bm{\Theta}_\tau)$, $Q(\bm{A})$ and $Q(\bm{B})$, respectively. Importantly, $O_I$ appears in the variational distribution because observations in the future are treated as hidden variables.

Finally, we follow \citet{millidge2020expected} in assuming that the agent aims to minimise the KL divergence between the approximate posterior depicting the state of the environment and a target (desired) distribution. Therefore, our framework allows for the specification of prior preferences over both future hidden states and future observations:
\begin{align}
\label{eq: target distribution}
V(O_{\mathbb{I}_t},S_{\mathbb{I}_t}) = \prod_{I \in \mathbb{I}_t} V(O_I)V(S_I)
\end{align}
where the individual factors are defined as:
\begin{align}
V(O_I) = \text{Cat}(\bm{C}_O),& &V(S_I) = \text{Cat}(\bm{C}_S).
\end{align}
Importantly, by specifying the value of future observations and states, $\bm{C}_O$ and $\bm{C}_S$ play a similar role to the vector $\bm{C}$ in active inference, i.e., they specify which observations and hidden states are rewarding.

To sum up, this framework is defined using three distributions: the prior defines the agent's beliefs before sampling any observation; the posterior is an updated version of the prior which takes into account past observations made by the agent; finally, the target distribution encodes the agent's prior preferences in terms of future observations and hidden states.

\begin{figure}[H]
	\begin{center}
	\includegraphics[scale=1]{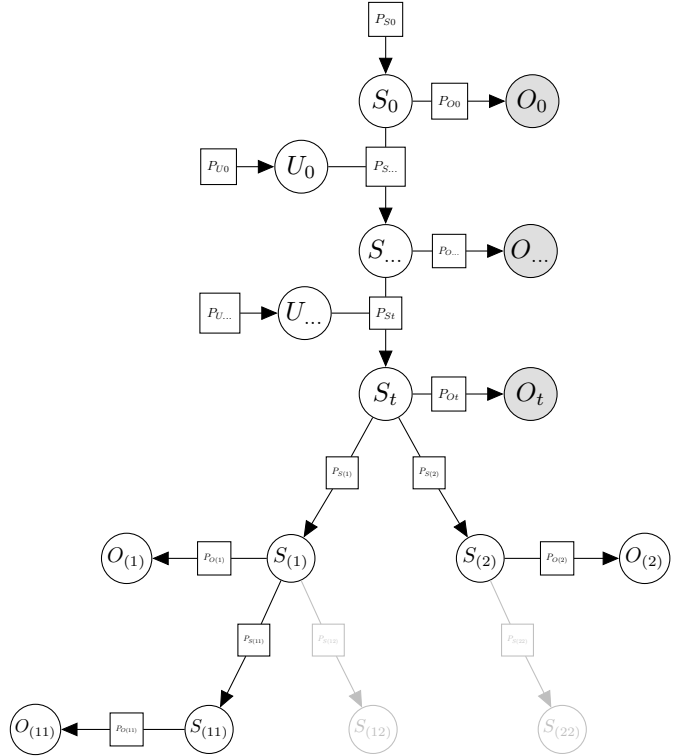}
 	\end{center}
\vspace{-0.25cm}
    \caption{
This figure illustrates the new expandable generative model allowing planning under active inference. The future is now a tree like generative model whose branches correspond to the policies considered by the agent. As we will see, these branches can be dynamically expanded during planning. Here, the nodes in light gray represent possible expansions of the current generative model. For the sake of clarity, the random tensor $\bm{A}$, $\bm{B}$, $\bm{\Theta}_\tau$ and $\bm{D}$ are not illustrated, i.e., Dirichlet priors over those random tensors are not shown.}
    \label{fig:AITS}
\end{figure}

\subsection{Bayesian belief updates} \label{sec:updates}

In this section, we focus on the set of update equations used to perform approximate Bayesian inference. These update equations rely on variational message passing as presented in Section \ref{sec:vmp}, see \citet{AI_VMP} as well as \citet{VMP_TUTO} for details. A key strength of the message passing approach is the capacity to derive and implement these updates within an automatic and modular toolbox \citep{Simul_AI,DBLP:journals/ijar/CoxLV19}, which in a way similar to automatic differentiation alleviates the final user from the burden of manually deriving complex update equations for each new generative model. To simplify our notation, we use two operators $\otimes$ and $\odot$ that we call generalized outer and inner product, respectively. The generalized outer product creates an $N$ dimensional tensor from $N$ vectors, while the generalized inner product performs a weighted average over one dimension of an $N$ dimensional array, cf Appendix A for details. Using these notations, the first set of update equations are given by:
\begin{align}
Q^*(\bm{D}) &= \text{Dir}\big(\bm{\hat{d}}\big) \,\quad \text{ where } \quad \bm{\hat{d}} = \bm{d} \,+ \bm{\hat{D}}_0 {\color{white}\sum_{\tau}^{t}}\\
Q^*(\bm{A}) &= \text{Dir}\big(\bm{\hat{a}}\big) \,\quad \text{ where }\quad \bm{\hat{a}} = \bm{a} + \sum_{\tau = 0}^t \otimes \Big[ \bm{\hat{D}}_\tau, \bm{o}_\tau\Big]\\
Q^*(\bm{B}) &= \text{Dir}\big(\bm{\hat{b}}\big) \,\quad \text{ where } \quad \bm{\hat{b}} = \bm{b} + \sum_{\tau = 1}^t \otimes \Big[ \bm{\hat{D}}_{\tau - 1}, \bm{\hat{\Theta}}_{\tau - 1}, \bm{\hat{D}}_{\tau}\Big]\\
Q^*(\bm{\Theta}_\tau) &= \text{Dir}\big(\bm{\hat{\theta}}_\tau\big) \,\,\, \text{ where } \quad \bm{\hat{\theta}}_\tau = \bm{\theta}_\tau \, + \bm{\hat{\Theta}}_\tau
\end{align}
where $\bm{o}_\tau$ is the observation made at time $\tau$. Furthermore, this first set of equations count (probabilistically) the number of times, an initial hidden state has been observed, an action has been performed, a state has generated a particular observation or an action has led to the transition between two consecutive hidden states. For example, the posterior parameters $\bm{\hat{a}}$ are computed by adding $\sum_{\tau = 0}^t \otimes [ \bm{\hat{D}_\tau}, \bm{o}_\tau]$ (i.e., the number of times a state-observation pair has been observed during this trial) to the prior parameters $\bm{a}$ (i.e., the number of times this same pair has been observed during previous trials). The equations for belief updates are given by:
\begin{align}
Q^*(O_I) &= \sigma \big( \bm{\mathring{A}} \odot \bm{\hat{D}}_I \big){\color{white}\sum_{\tau}^{t}}\\
Q^*(S_I) &= \sigma \Big( \bm{\mathring{A}} \odot \bm{\hat{E}}_I + \bm{\mathring{B}}_I \odot \bm{\hat{D}}_\IdMLast{I} + \sum_{J \in \text{ch}_I} \bm{\mathring{B}}_J \odot \bm{\hat{D}}_J\Big) {\color{white}\sum_{\tau}^{t}}\\
Q^*(U_\tau) &= \sigma \big( \bm{\mathring{\Theta}} + \bm{\mathring{B}} \odot [ \bm{\hat{D}}_{\tau}, \bm{\hat{D}}_{\tau + 1}] \big) {\color{white}\sum_{\tau}^{t}}
\end{align}
\vspace{-1.5cm}
\begin{align}
\quad \quad Q^*(S_\tau) = \sigma\Big(&[\tau = 0] \bm{\mathring{D}}_\tau \,\,+\,\, [\tau \neq 0] \bm{\mathring{B}} \odot [\bm{\hat{D}}_{\tau - 1}, \bm{\hat{\Theta}}_{\tau - 1}]{\color{white}\sum^{t}}\nonumber\\
+ &\bm{\mathring{A}} \odot \bm{o}_\tau {\color{white}\sum_{\tau}^{t}}\nonumber\\
+ &[\tau = t] \sum_{J \in \text{ch}_t} \bm{\mathring{B}}_J \odot \bm{\hat{D}}_{J} \,\,+\,\, [\tau \neq t] \bm{\mathring{B}} \odot [ \bm{\hat{D}}_{\tau + 1}, \bm{\hat{\Theta}}_{\tau}]\Big)
\end{align}
where $\sigma(\bigcdot)$ is the softmax function, $\text{ch}_t$ are the children (states) of the current states $S_t$, $\text{ch}_I$ are the children (states) of the states $S_I$, $[\text{predicate}]$ is an indicator function returning one if the predicate is true and zero otherwise, and the definition of $\bm{\mathring{A}}$, $\bm{\mathring{B}}$, $\bm{\mathring{D}}$ and $\bm{\mathring{\Theta}}_\tau$ are given in Table \ref{tab:message_passing_notation}. Note that thanks to the operators $\otimes$ and $\odot$, the perception (i.e., state-estimation) equations can be intuitively understood as a sum of messages, where each message from a factor to a variable is the average over all dimensions except the dimension of the variable, e.g., the message ($\bm{\mathring{A}} \odot \bm{o}_0$) from $P_{o_0}$ to $S_0$ is the vector obtained by weigthing the rows of $\bm{\mathring{A}}$ by the elements of $\bm{o}_0$. Importantly, the above update equations are almost identical to the ones used in standard active inference, and thus can be implemented efficiently. Indeed, most of the computation required is about addition of matrices $O(n^2)$ and multiplication of matrices $O(n^3)$, or their higher dimensional counterparts.

\begin{table}[h]
\begin{center}
\begin{tabular}{cl}
\hline
\\[-0.15cm]
Notation & Meaning \\
\\[-0.15cm]
\hline
\hline
\\[-0.4cm]
$\langle f(X) \rangle_{P_X} \delequal \mathbb{E}_{P_X}[f(X)]$ & The expectation of $f(X)$ over $P_X$\\
\\[-0.4cm]
\hline
\\[-0.4cm]
$\psi(\bigcdot)$ & The digamma function\\
\\[-0.4cm]
\hline
\\[-0.4cm]
$\bm{\mathring{\Theta}}_{\tau}(i) = \langle \ln \bm{\Theta}_{\tau}(i) \rangle_{{Q_\Theta}_\tau} = \psi\big(\bm{\hat{\theta}}_\tau(i)\big) - \psi\big(\sum_k \bm{\hat{\theta}}_\tau(k)\big)$ & The expected logarithm of $\bm{\Theta}_\tau$\\
\\[-0.4cm]
\hline
\\[-0.4cm]
$\bm{\mathring{D}}(i) = \langle \ln \bm{D}(i) \rangle_{Q_D} = \psi\big(\bm{\hat{d}}(i)\big) - \psi\big(\sum_k \bm{\hat{d}}(k)\big)$ & The expected logarithm of $\bm{D}$\\
\\[-0.4cm]
\hline
\\[-0.4cm]
$\bm{\mathring{A}}(i,j) = \langle \ln \bm{A}(i,j) \rangle_{Q_A} = \psi\big(\bm{\hat{a}}(i,j)\big) - \psi\big(\sum_k \bm{\hat{a}}(k,j)\big)$ & The expected logarithm of $\bm{A}$\\
\\[-0.4cm]
\hline
\\[-0.4cm]
$\bm{\mathring{B}}(i,j,u) = \langle \ln \bm{B}(i,j,u) \rangle_{Q_B} = \psi\big(\bm{\hat{b}}(i,j,u)\big) - \psi\big(\sum_k \bm{\hat{b}}(k,j,u)\big)$ & The expected logarithm of $\bm{B}$\\
\\[-0.4cm]
\hline
\\[-0.4cm]
$\bm{\bar{A}}(i,j) = \langle \bm{A}(i,j) \rangle_{Q_A} = \frac{\bm{\hat{a}}(i,j)}{\sum_k \bm{\hat{a}}(k,j)}$ & The expectation of $\bm{A}$\\
\\[-0.4cm]
\hline
\\[-0.4cm]
$\bm{\bar{B}}(i,j,u) = \langle \bm{B}(i,j,u) \rangle_{Q_B} = \frac{\bm{\hat{b}}(i,j,u)}{\sum_k \bm{\hat{b}}(k,j,u)}$ & The expectation of $\bm{B}$\\
\\[-0.4cm]
\hline
\end{tabular}
\caption{Update equations notation. Note that Appendix D provides a proof for $\bm{D}$.}\label{tab:message_passing_notation}
\end{center}
\vspace{-0.75cm}
\end{table}

\subsection{Planning as structure learning}

In this section, we frame planning as a form of structure learning where the structure of the generative model is modified dynamically. This method is greatly inspired by the Monte Carlo tree search literature, c.f., Section \ref{sec:MCTS} for details.

\subsubsection{Selection of the node to be expanded} \label{sec:selection}

The first step of planning is to select a node to be expanded. The selection process starts at the root node, if the root node still has unexplored children, then one of them is selected. Otherwise, the child node maximizing the $UCT$ criterion, where the average reward is replaced by minus the average EFE, is selected, i.e., the selected node maximises:
\begin{align}
UCT_J = \underbrace{- \bar{g}_J}_{\text{exploitation}} +\quad \underbrace{C_p \sqrt{\frac{\ln n}{n_J}}}_{\text{exploration}},
\end{align}
where $J$ is a multi-index, $n$ is the number of times the root node has been visited, $n_J$ is the number of times the child corresponding to the multi-index $J$ was selected, and $\bar{g}_J$ is the average cost received when selecting the child $S_J$. The $UCT$ criterion can be understood as a trade-off between exploitation and exploration at the tree level, which is different to the exploitation and exploration dilemma at the model level. This dilemma is handled by the EFE. Also, the notion of cost in the above equation can be defined in many ways and will be the subject of Section \ref{sec:quality}. For our purposes, the cost will be equal, or similar, to the expected free energy, which means that the expected free energy drives structure learning. When a root's child is selected, it becomes the new root in the above procedure, which is iterated until a leaf node is reached. 

\subsubsection{Dynamical expansion of the generative model} \label{sec:expansion}

Let $S_\IdMLast{I}$ denotes the leaf node selected for expansion. When $S_\IdMLast{I}$ has been selected, the structure of the generative model needs to be modified by expanding all possible actions from that node. For each action, we expand the generative model by adding a future hidden state whose prior distribution is given by

\begin{align}
P(S_I|S_\IdMLast{I}) = \text{Cat}(\bm{\bar{B}}_I),
\end{align}
where $\bm{\bar{B}}_I$ is the matrix corresponding to the last action that led to $S_I$. Finally, we expand the (future) observation associated with the new hidden state $S_I$, whose distribution is:

\begin{align}
P(O_I|S_I) = \text{Cat}(\bm{\bar{A}}).
\end{align}

To sum up, the expansion step is adding two random variables ($S_I$ and $O_I$) to the generative model, i.e. the generative model becomes bigger, and $I$ is added to the set of all non-empty multi-indices already expanded by the tree search ($\mathbb{I}_t$). The prior distributions over those newly added random variables (i.e. $S_I$ and $O_I$) are defined using the matrices $\bm{\bar{B}}_I$ and $\bm{\bar{A}}$, which effectively predict the future states and observations. After the expansion step, the posterior distribution over $S_I$ and $O_I$ needs to be computed. At least two kinds of inference strategies can be used. The first---global inference---performs variational message passing over the entire generative model, while the second---local inference---only iterates the update equations of the newly expanded nodes, i.e., $S_I$ and $O_I$, until convergence to the variational free energy minimum.

\subsubsection{Cost evaluation of the expanded nodes} \label{sec:quality}

After expanding the model structure, we need to compute the cost of the newly expanded node $S_I$. As explained in Section \ref{sec:selection}, the cost of $S_I$ will influence the probability of expanding $S_I$ during future planning iterations. In active inference, the classic objective of planning is the expected free energy as defined in Section \ref{ssec:GM}, i.e.,
\begin{align}
g^{classic}_I \delequal D_{\mathrm{KL}}[Q(O_I)||V(O_I)]\,\, +\,\, \mathbb{E}_{Q(S_I)}[\text{H}[P(O_I | S_I)]]
\end{align}
where $g^{classic}_I$ trades off risk (first summand) and ambiguity (second summand). Alternatively, one could follow Section 5 of \citet{millidge2020expected} and define the cost of $S_I$ using the free energy of the expected future:
\begin{equation}
    g^{feef}_I = \kl{Q(O_I, S_I)}{V(O_I, S_I)}
\end{equation}
where $V(O_I, S_I)$ is the target distribution over states and observations. The target distribution $V(O_I, S_I)$ generalises the $\bm{C}$ matrix in Friston's model by specifying prior preferences over both future observations and future states. Also, this formulation of the cost speaks to the notion of KL divergence minimization proposed by \citet{hafner2020action}.

Furthermore, due to the mean-field approximation of the posterior \eqref{eq: mean field approx posterior} and the factorised form of the target distribution \eqref{eq: target distribution}, the expression of the cost simplifies to

\begin{align}
g^{pcost}_I \delequal \kl{Q(S_I)}{V(S_I)} + \kl{Q(O_I)}{V(O_I)}.
\end{align}

Intuitively, the pure cost ($g^{pcost}_I$) measures how different the predicted future hidden states and observations are from desired states and observations. In future research, it might be interesting to compare the performance and behaviour of $g^{pcost}_I$, $g^{feef}_I$ and $g^{classic}_I$ empirically and theoretically. 

Note that in MCTS, the evaluation of a node's quality is done by performing virtual roll-outs, while $g^{pcost}_I$, $g^{feef}_I$ and $g^{classic}_I$ are not. If we let $g_I$ be any of those criteria, then we can improve our estimate of the cost, by computing $g_I^{average}$, i.e., the average cost over $N$ roll-outs of size $K$. Algorithm \ref{algo:roll-out} presents the pseudo code used to estimate $g_I^{average}$.

\begin{algorithm}[H]
\label{algo:roll-out}
\SetKwInOut{Input}{Input}
\SetKwFor{RepTimes}{repeat}{times}{end}
\SetAlgoLined
\Input{$N$ the number of virtual roll-outs, $K$ the maximal length of each roll-out.}
 $g_I^{average} \leftarrow 0$ \tcp*{Initialize roll-out estimate to zero}
 \RepTimes{$N$} {
 $g_I^{rollout} \leftarrow g_I$ \tcp*{Initial cost equals cost of $S_I$}
 \For{$i\leftarrow 1$ \KwTo $K$} {
  sample a random action $U_i$ uniformly from the set of unexplored actions\;
  perform the expansion of the current node using $U_i$ (Section \ref{sec:expansion})\;
  perform inference on the newly expanded nodes (Section \ref{sec:updates})\;
  $g_I^{rollout} \leftarrow g_I^{rollout} + g_{J}$ \tcp*{$J$ corresponds to last expanded node}
 }
 $g_I^{average} \leftarrow g_I^{average} + g_I^{rollout}$\;
 }
 $g_I^{average} \leftarrow g_I^{average} / N$\;
 \caption{Estimation of $g_I^{average}$}
\end{algorithm}

\subsubsection{Propagation of the node cost} \label{sec:propagation}

In this section, we let $\bm{G}^{aggr}_L$ be a variable that contains the total cost of the node $S_L$, where $L$ could be any multi-index. According to the previous section, we let $g_L$ be any of the following evaluation criteria $g^{pcost}_L$, $g^{feef}_L$ and $g^{classic}_L$. Initially, $\bm{G}^{aggr}_L$ equals $g_L$. Also, we let $S_K$ be the node that was selected for expansion, and let $S_I$ be an arbitrary hidden state expanded from $S_K$. The cost of the newly expanded node(s) can be propagated either forward or backward. The forward propagation (towards the leaves) leads to the following equation:
\begin{align}
\bm{G}^{aggr}_I \leftarrow g_I + \bm{G}^{aggr}_K,
\end{align}
where here $\bm{G}^{aggr}_K$ is the aggregated cost of the parent of $S_I$. Importantly, the symbol $\leftarrow$ refers to a programming-like assignment (i.e., an incremental update) performed each time the tree is expanded. The backward propagation (towards the root) leads to:
\begin{equation}\label{eq:backprop}
\bm{G}^{aggr}_J \leftarrow \bm{G}^{aggr}_J + g_I \quad \forall J \in \mathbb{A}_I
\end{equation}
where $\mathbb{A}_I$ corresponds to all ancestors of the newly expanded node $S_I$. We will see in Section \ref{sec:aits_ai_sai} that these strategies respectively relate to active inference and sophisticated inference \citep{Sophisticated_INF}. Finally, since the agent is free to choose any action, we can back-propagate the (locally) minimum cost, i.e.,
\begin{equation}\label{eq:minbackprop}
\bm{G}^{aggr}_J \leftarrow \bm{G}^{aggr}_J + \min_{a \in \{1, ..., |U|\}} g_{K::a} \quad \forall J \in \mathbb{A}_I,
\end{equation}
where $K::a$ is a multi-index obtained from $K$ by adding the action $a$ to the sequence of actions described by $K$. In all cases, the propagation step updates the counter $n_J$ associated with each ancestor $S_J$ of the newly expanded hidden state $S_I$; this counts the number of times the node $S_J$ has been explored (exactly as in MCTS). This counter will be used for action selection, as well as for the computation of the average cost of $S_J$---$\bar{g}_J$---that was left undefined by Section \ref{sec:selection}. Formally, $\bar{g}_J$ is given by:
\begin{align}
\bar{g}_J = \frac{1}{n_J}\bm{G}^{aggr}_J.
\end{align}

\begin{remark}
The forward propagation of the cost presented above will only be used for theoretical purpose in Section \ref{sec:aits_ai_sai}. Practical implementation of BTAI should use the backward schemes.
\end{remark}

\subsection{Action selection} \label{sec:action_selection}

The planning procedure presented in the previous section ends after a pre-specified amount of time has elapsed or when a sufficiently good policy has been found. When the planning is over, the agent needs to choose an action to act in its environment. In a companion paper \citep{AITS_PRACTICE} that presents empirical results of BTAI, the actions are sampled from $\sigma(-\gamma \frac{g}{N})$, where $\sigma(\bigcdot)$ is a softmax function, $\gamma$ is a precision parameter, $g$ is a vector whose elements correspond to the cost of the root's children and $N$ is a vector whose elements correspond to the number of visits of the root's children. Importantly, actions with low average cost are more likely to be selected than actions with high average cost.

Alternative approaches to action selection \citep{6145622} could be studied. For example, one could imagine sampling actions from a categorical distribution with parameter $\sigma(N)$, where $N$ is a vector containing the $n_J$ of all children of the root node. Or, we could select the action corresponding to the root's child with the highest number of explorations $n_J$. The fact that it has been visited more often means that is has a lower cost overall. If there were a tie between several actions, the action with the lowest cost would be selected. The study of these strategies is left to future research.

\subsection{Action-perception cycle with tree search} \label{sec:act-per-cycle}

In active inference, the action-perception cycle realises an active inference agent in an infinite loop \citep{Simul_AI}. Each loop iteration begins with the agent sampling an observation from the environment. The observation is used to perform inference about the states and contingencies of the world, e.g., an impression on the retina might be used to reconstruct a three dimensional scene with a representation of the objects that it contains. Then, planning is performed by inferring the consequences of alternative action sequences. Importantly, only a subset of all possible action sequences are evaluated, due to the dynamical expansion of the generative model. Finally, the agent selects an action to perform in the environment by sampling a softmax function of minus the average cost weighted by the precision parameter $\gamma$, i.e., $\sigma(-\gamma \frac{g}{N})$. Therefore, actions with low average cost are more likely to be selected than actions with high average cost. We summarise our method using pseudo-code in Algorithm \ref{algo: AITS}.

\begin{algorithm}[H]
\label{algo: AITS}
\SetAlgoLined
 \While{end of trial not reached}{
  sample an observation from the environment\;
  perform inference using the observation (Section \ref{sec:updates})\;
  \While{maximum planning iteration not reached}{
   select a node to be expanded (Section \ref{sec:selection})\;
   perform the expansion of the node (Section \ref{sec:expansion})\;
   perform inference on the newly expanded nodes (Section \ref{sec:updates})\;
   evaluate the cost of the newly expanded nodes (Section \ref{sec:quality})\;
   propagate the cost of the nodes through the tree, either forward or backward (Section \ref{sec:propagation})\;
  }
  select an action to be performed (Section \ref{sec:action_selection})\;
  execute the action in the environment leading to a new observation\;
 }
 \caption{Action-perception cycle with tree search}
\end{algorithm}

\section{Connection between BTAI, active inference and sophisticated inference} \label{sec:aits_ai_sai}

In this section, we explore the relationship between BTAI, active inference (AcI) and sophisticated inference (SI). We show that BTAI is a class of algorithms that generalizes AcI and is related to SI. To do so, we focus on the ``cost" of a policy for each method. In addition, we need to introduce the notion of localized and aggregated cost. The localized cost of a node $S_I$, denoted $\bm{G}_I^{local}$, is the cost of $S_I$ in and of itself, i.e., without any consideration of the cost of past or future states. The aggregated cost of a node $S_I$, denoted $\bm{G}_I^{aggre}$, is the cost of $S_I$ when taking into account either the cost of future states that can be reached from $S_I$ (which is the case in SI) or the cost of the past states that an agent has to go through in order to reach $S_I$ (which is the case in AcI).

\subsection{Active inference}

The full framework of active inference was described in Section \ref{sec:ai}. This section focuses on expressing the expected free energy in a recursive form that highlights the relationship between BTAI and AcI. We start by defining the notion of localized and aggregated EFE with Definitions \ref{def:localized_efe} and \ref{def:aggregated_efe}, respectively. Then, we show that in active inference (under some assumptions described below), the aggregated EFE of a policy of size $N$ is given by the aggregated EFE of a policy of size $N - 1$ plus the localized EFE received at time $t + N$.

In active inference, a policy is a sequence of actions $\pi = (U_t, U_{t+1}, ..., U_{T - 1})$, where $T$ is the time horizon of planning, and for convenience, $\pi_{N}$ denotes a policy of size $N$, obtained by selecting the first $N$ actions of the policy $\pi$, i.e., $\pi_{N} = (U_t, U_{t+1}, ..., U_{t + N - 1})$ with $N \leq T - t$. Recall from Section \ref{sec:ai}, that (in active inference) the expected free energy of a policy is given by:
\begin{align}\label{eq:efe_act_inf}
\bm{G}(\pi) = \sum_{\tau=t + 1}^T \bm{G}(\pi, \tau) = \sum_{\tau=t + 1}^T \Bigg[ D_{\mathrm{KL}}[Q(O_\tau|\pi)||P(O_\tau)]\,\, +\,\, \mathbb{E}_{Q(S_\tau|\pi)}[\text{H}[P(O_\tau | S_\tau)]]\Bigg].
\end{align}
If instead of letting $\tau$ range from $t + 1$ to $T$, we let $N$ range from $1$ to $T - t$, then Equation \ref{eq:efe_act_inf} can be re-written as:
\begin{align}
\bm{G}(\pi) = \sum_{N=1}^{T-t} \bm{G}(\pi,t+N) = \sum_{N=1}^{T-t} \Bigg[ D_{\mathrm{KL}}[Q(O_{t+N}|\pi)||P(O_{t+N})]\,\, +\,\, \mathbb{E}_{Q(S_{t+N}|\pi)}[\text{H}[P(O_{t+N} | S_{t+N})]]\Bigg].
\end{align}
Additionally, under the assumption that the probability of observations and states are independent of future actions, i.e., that $\forall j \in \mathbb{N}_{>0}, Q(O_{t+i}|\pi_{i}) \approx Q(O_{t+i}|\pi_{i+j})$ and $\forall j \in \mathbb{N}_{>0}, Q(S_{t+i}|\pi_{i}) \approx Q(S_{t+i}|\pi_{i+j})$, $\pi$ can be replaced by $\pi_N$ in the RHS of the above equation, leading to:
\begin{align}\label{eq:efe_rearranged_4242}
\bm{G}(\pi) = \sum_{N=1}^{T-t} \bm{G}(\pi_N,t+N) = \sum_{N=1}^{T-t} \Bigg[ D_{\mathrm{KL}}[Q(O_{t+N}|\pi_N)||P(O_{t+N})]\,\, +\,\, \mathbb{E}_{Q(S_{t+N}|\pi_N)}[\text{H}[P(O_{t+N} | S_{t+N})]]\Bigg].
\end{align}
Importantly, the elements of the above summation constitute the localized cost presented in Definition \ref{def:localized_efe}.

\begin{definition} \label{def:localized_efe}
We define the \textbf{localized cost} received at time $t+N$ after selecting policy $\pi_N$ as:
\begin{align}
\bm{G}^{local}_{\pi_N} = \bm{G}(\pi_N,t+N) = D_{\mathrm{KL}}[Q(O_{t+N}|\pi_N)||P(O_{t+N})]\,\, +\,\, \mathbb{E}_{Q(S_{t+N}|\pi_N)}[\text{H}[P(O_{t+N} | S_{t+N})]].
\end{align}
\end{definition}
Importantly, the localized cost quantifies the amount of risk and ambiguity received by the agent at time step $t+N$, assuming that it will follow the policy $\pi_N$. We now turn to the notion of aggregated cost of a policy of size $N$. Definition \ref{def:aggregated_efe} states that the aggregated cost of a policy is defined recursively. Indeed, by definition, a policy of size zero has an aggregated cost of zero, and then, the aggregated cost of a policy $\pi_N$ (of size $N$) is equal to the the aggregated cost of $\pi_{N - 1}$ (of size $N - 1$) plus the localized cost received at time $t+N$.

\begin{definition} \label{def:aggregated_efe}
We define the \textbf{aggregated cost} of a policy $\pi_N$ of size $N$ as:
\begin{align}
\bm{G}^{aggre}_{\pi_N} = \left\{ \begin{matrix}
0 & \text{if } N = 0\\
\bm{G}^{aggre}_{\pi_{N - 1}} + \bm{G}^{local}_{\pi_N} & \text{otherwise}\\
\end{matrix}\right. .
\end{align}
\end{definition}

Equipped with Definitions \ref{def:localized_efe} and \ref{def:aggregated_efe}, we are now ready to state and prove Theorem \ref{th:abstract_AI} using the two Lemmas of Appendix E.

\begin{theorem}\label{th:abstract_AI}
Under the assumption that the probability of observations and states are independent of future actions, i.e., $\forall j \in \mathbb{N}_{>0}, Q(O_{t+i}|\pi_{i}) \approx Q(O_{t+i}|\pi_{i+j})$ and $\forall j \in \mathbb{N}_{>0}, Q(S_{t+i}|\pi_{i}) \approx Q(S_{t+i}|\pi_{i+j})$, the expected free energy can be written as:
\begin{align}\label{abstract_AI}
\bm{G}(\pi_N) \approx \bm{G}_{\pi_{N}}^{aggre} = \bm{G}_{\pi_{N-1}}^{aggre} + \bm{G}_{\pi_{N}}^{local}.	
\end{align}
\end{theorem}

\begin{proof}
This proof is based on two lemmas demonstrated in Appendix E. Note that in active inference the expected free energy is defined as:
\begin{align}
\bm{G}(\pi) = \sum_{\tau=t + 1}^{T} \bm{G}(\pi, \tau).
\end{align}
Let $N$ denote the size of the policy $\pi$, i.e. $N = T - t$. Note that because $\pi$ is of size $N$, then by definition $\pi = \pi_N$, and the above equation can be re-written as:
\begin{align}
\bm{G}(\pi) = \bm{G}(\pi_N) = \sum_{\tau=t + 1}^{t+N} \bm{G}(\pi_{N}, \tau).
\end{align}
Expanding the summation and using Definition \ref{def:localized_efe}:
\begin{align}
\bm{G}(\pi_N) &= \sum_{\tau=t + 1}^{t+N-1} \bm{G}(\pi_N, \tau) + \bm{G}(\pi_N,t+N)\\
&= \sum_{\tau=t + 1}^{t+N-1} \bm{G}(\pi_N, \tau) + \bm{G}^{local}_{\pi_N}.
\end{align}
If, instead of letting $\tau$ range from $t + 1$ to $t+N-1$, we let $i$ range from $1$ to $N-1$, then the above equation can be re-written as:
\begin{align}
\bm{G}(\pi_N) = \sum_{i=1}^{N-1} \bm{G}(\pi_N, t + i) + \bm{G}^{local}_{\pi_N}.
\end{align}
Note that $\forall i \in \{1, ..., N-1\}, N > i$, and thus there exists a $k_i \in \mathbb{N}_{>0}$ such that $\forall i \in \{1, ..., N-1\}, N = i + k_i$. Therefore, we replace $N$ by $i + k_i$ in the above summation:
\begin{align}
\bm{G}(\pi_N) = \sum_{i=1}^{N-1} \bm{G}(\pi_{i+k_i}, t + i) + \bm{G}^{local}_{\pi_N}.
\end{align}
Lemma \ref{lemma_51} tells us that under the assumption that the probability of observations and states are independent of future actions, $\forall k_i \in \mathbb{N}_{>0}, \bm{G}(\pi_{i+k_i}, t+i) \approx \bm{G}(\pi_i, t+i)$, which allows us to remove the $k_i$ to get:
\begin{align}
\bm{G}(\pi_N) \approx \sum_{i=1}^{N-1} \bm{G}(\pi_i, t + i) + \bm{G}^{local}_{\pi_N}.
\end{align}
Finally, Lemma \ref{lemma_52} states that $\sum_{i=1}^{N-1} \bm{G}(\pi_i, t + i) = \bm{G}^{aggre}_{\pi_{N-1}}$, and thus:
\begin{align}
\bm{G}(\pi_N) \approx \bm{G}^{aggre}_{\pi_{N-1}} + \bm{G}^{local}_{\pi_N} \delequal \bm{G}^{aggre}_{\pi_N}.
\end{align}
The above equation will be used in Section \ref{ssec:BTAI_generalize_AI} to show that BTAI generalizes active inference.
\end{proof}

\subsection{Sophisticated inference}

Sophisticated inference \citep{Sophisticated_INF} is a new type of active inference that defines the EFE recursively from the time horizon backward. Intuitively, the agent does not simply ask ``what would happen if I did that", but instead wonders ``what would I believe about what would happen if I did that". In other words, the agent is exhibiting a form of sophistication, which refers to the fact of having beliefs about one's own or another's beliefs. \citet{Sophisticated_INF} also replaced variational message passing by an alternative inference scheme called Bayesian Filtering \citep{BAYESIAN_FILTERING}. While the change of inference method is of little relevance to us here, the recursive definition of the EFE is at the core of this section. As explained in Section 4.3 of \citet{dacosta2020relationship}, the (recursive) EFE of a Markov decision process is given by:
\begin{align}
G(U_{T-1},S_{T-1})&= \kl{Q(S_{T}|U_{T-1},S_{T-1})}{V(S_{T})}\\
G(U_\tau,S_\tau)&= \kl{Q(S_{\tau+1}|U_\tau,S_\tau)}{V(S_{\tau+1})} + \mathbb{E}_{Q(U_{\tau+1},S_{\tau+1}|U_\tau,S_\tau)}[G(U_{\tau+1},S_{\tau+1})]\label{eq:recu_def_of_efe_in_SI}
\end{align}
where $U_\tau$ and $S_\tau$ are the action and state at time $\tau$, and $V(S_{\tau})$ is the target (i.e., desired) distribution over states at time $\tau$. Using our terminology of localized and aggregated cost, this can be rewritten as:
\begin{align}
\underbrace{G(U_{T-1},S_{T-1})}_{\bm{G}^{aggre}(U_{T-1},S_{T-1})} &= \,\,\underbrace{\kl{Q(S_{T}|U_{T-1},S_{T-1})}{V(S_{T})}}_{\bm{G}^{local}(U_{T-1},S_{T-1})} \\
\underbrace{G(U_\tau,S_\tau)}_{\bm{G}^{aggre}(U_{\tau},S_{\tau})} &= \underbrace{\kl{Q(S_{\tau+1}|U_\tau,S_\tau)}{V(S_{\tau+1})}}_{\bm{G}^{local}(U_{\tau},S_{\tau})} + \mathbb{E}_{Q(U_{\tau+1},S_{\tau+1}|U_\tau,S_\tau)}[\underbrace{G(U_{\tau+1},S_{\tau+1})}_{\bm{G}^{aggre}(U_{\tau+1},S_{\tau+1})}]
\end{align}
Put simply, the aggregated cost of taking action $U_\tau$ in state $S_\tau$ can be computed by summing the localized cost at time step $\tau$ and the expected aggregated cost at time step $\tau + 1$, i.e.,
\begin{align}\label{abstract_SI}
\bm{G}^{aggre}(U_{\tau},S_{\tau}) = \bm{G}^{local}(U_{\tau},S_{\tau}) + \mathbb{E}_{Q(U_{\tau+1},S_{\tau+1}|U_\tau,S_\tau)}[\bm{G}^{aggre}(U_{\tau+1},S_{\tau+1})]
\end{align}
Note that for $\tau = T - 1$ the second term vanishes because future states beyond the temporal horizon are ignored, and thus $\bm{G}^{aggre}(U_{T-1},S_{T-1}) = \bm{G}^{local}(U_{T-1},S_{T-1})$. Also, the above equation will be useful in Section \ref{ssec:BTAI_relate_to_SI} to show that BTAI is related to sophisticated inference.

\begin{remark}
The recursive aspect of Equation \ref{eq:recu_def_of_efe_in_SI} is deeply related to dynamic programming and the interested reader is referred to \citet{dacosta2020relationship} for details about this relationship.
\end{remark}

\subsection{Branching Time Active Inference (BTAI)}

In BTAI, the (localized) cost of the hidden state $S_I$ is defined as $\bm{G}^{local}_I = g_I$, where $g_I$ can be equal to $g^{classic}_I$, $g^{feef}_I$ or $g^{pcost}_I$, and there are two ways of computing the aggregated cost of $S_I$. We can either propagate the localized cost towards the leaves (forward):
\begin{align}
g_I \leftarrow g_I + g_\IdMLast{I}
\end{align}
where the $g_\IdMLast{I}$ is the cost of the parent of $S_I$. Alternatively, we can back-propagate the cost towards the root
\begin{align}
g_J \leftarrow g_J + g_I \quad \forall J \in \mathbb{A},
\end{align}
where $\mathbb{A}$ corresponds to all ancestors of the newly expanded node $S_I$. 

\subsection{BTAI as a generalisation of active inference}\label{ssec:BTAI_generalize_AI}

To understand the relationship between BTAI and active inference, we need to focus on the forward propagation of the cost where the cost is given by $g^{classic}_I$. Recall that the update for forward propagation is given by:
\begin{align}
g^{classic}_I \leftarrow g^{classic}_I + g^{classic}_\IdMLast{I},
\end{align}
where the $g_\IdMLast{I}$ is the cost of the parent of $S_I$, i.e., the parent of the newly expanded node. This equation tells us that the aggregated cost of $S_I$ is equal to the localized cost of $S_I$ plus the aggregated cost of $S_\IdMLast{I}$, i.e.,
\begin{align} \label{AITS_as_AI}
\underbrace{g^{classic}_I}_{\bm{G}^{aggre}_I} \leftarrow \underbrace{g^{classic}_I}_{\bm{G}^{local}_I} + \underbrace{g^{classic}_\IdMLast{I}}_{\bm{G}^{aggre}_\IdMLast{I}} \quad \Leftrightarrow \quad \bm{G}^{aggre}_I = \bm{G}^{local}_I + \bm{G}^{aggre}_\IdMLast{I},
\end{align}
but, then, we also recall \eqref{abstract_AI}, i.e.,
\begin{align}
\bm{G}_{\pi_{N}}^{aggre} \approx \bm{G}_{\pi_{N-1}}^{aggre} + \bm{G}_{\pi_{N}}^{local} \quad \Leftrightarrow \quad \bm{G}_{\pi_{N}}^{aggre} \approx \bm{G}_{\pi_{N}}^{local} + \bm{G}_{\pi_{N-1}}^{aggre}.
\end{align}
The only difference between Equations \ref{abstract_AI} and \ref{AITS_as_AI} is notational. Indeed, in BTAI (Eq. \ref{abstract_AI}) a policy is represented by a multi-index denoting the sequence of actions selected, e.g., $I = (1, 2)$ corresponds to a policy of size two consisting of action one followed by action two. In contrast, in active inference, a policy is a sequence of actions, e.g., $\pi_{2} = (1, 2)$ corresponds to the same policy as the one described by $I$.

\subsection{Relationship between BTAI and sophisticated inference} \label{ssec:BTAI_relate_to_SI}

The relationship between BTAI and sophisticated inference is slightly more involved. The backward propagation equation, i.e.,
\begin{align}
g_J \leftarrow g_J + g_I \quad \forall J \in \mathbb{A},
\end{align}
tells us that when expanding a node $S_I$, we first need to compute its localized cost $g_I$ and then add $g_I$ to the aggregated cost of its ancestors $S_J$ where $J \in \mathbb{A}$. In other words, we can rewrite the backward propagation equation as the following: the aggregated cost of an arbitrary node $S_J$ will equal the sum of its localized cost $g_J$ and that of its descendants $\mathbb{D}_J$ that have already been evaluated
\begin{align}
\bm{G}_J^{aggre} = \bm{G}_J^{local} + \sum_{S_K \in \mathbb{D}_J} \bm{G}_K^{local},
\end{align}
where the descendants are the children, children of children, etc. We can further simplify this expression by grouping the summands by children of $S_J$. This leads us to:
\begin{align}\label{AITS_as_SI}
\bm{G}_J^{aggre} &= \bm{G}_J^{local} + \sum_{S_I \in \text{ch}_J} \underbrace{\left(\bm{G}_I^{local} + \sum_{S_K \in \mathbb{D}_I} \bm{G}_K^{local}\right)}_{\bm{G}_I^{aggre}} \\
&= \bm{G}_J^{local} + \sum_{S_I \in \text{ch}_J} \bm{G}_I^{aggre},
\end{align}
where $\text{ch}_J$ are the children of $S_J$ and $\mathbb{D}_I$ are the descendants of $S_I$. The above equation has clear similarities to Equation \eqref{abstract_SI}, which is, 
\begin{align}
\bm{G}^{aggre}(U_{\tau},S_{\tau}) = \bm{G}^{local}(U_{\tau},S_{\tau}) + \mathbb{E}_{Q(U_{\tau+1},S_{\tau+1}|U_\tau,S_\tau)}[\bm{G}^{aggre}(U_{\tau+1},S_{\tau+1})].
\end{align}
However, the second term of the RHS of \eqref{abstract_SI} is an expectation, while the second term of the RHS of \eqref{AITS_as_SI} is a summation over the children that have already been expanded. The expectation in \eqref{abstract_SI} is w.r.t.
\begin{align}
    Q(U_{\tau+1},S_{\tau+1}|U_\tau,S_\tau) &= Q(U_{\tau+1}|S_{\tau+1})Q(S_{\tau+1}|U_\tau,S_\tau),
\end{align}
where:
\begin{align}
    Q(U_{\tau+1}|S_{\tau+1}) > 0 &\Leftrightarrow U_{\tau+1} \in \argmins_{U \in \mathbb{U}} \bm{G}^{aggre}(U,S_{\tau+1}),
\end{align}
where $\mathbb{U}$ is the set of all possible actions, and argmins is defined as:
\begin{align}
M = \argmins_{U \in \mathbb{U}} f(U) \Leftrightarrow M = \Big\{ U \in \mathbb{U} \,\,\big|\,\, f(U) = \min_{U' \in \mathbb{U}} f(U')\Big\}.
\end{align}
This means that an action is assigned positive probability mass if and only if it minimises the aggregated cost at the next time point $\bm{G}^{aggre}(U_{\tau+1},S_{\tau+1})$ and a set is required, because multiple actions could have the same minimum cost. Note that if there is a unique minimum, then $Q(U_{\tau+1}|S_{\tau+1})$ will be a one-hot like distribution with a probability of one for the best action. 

To conclude, Equations \eqref{abstract_SI} and \eqref{AITS_as_SI} suggest that BTAI and SI share a similar notion of EFE, where the immediate (or localized) EFE is added to the future (or aggregated) EFE. Both BTAI and SI propagate the cost backward, however, in SI the aggregated EFE (i.e., the back-propagated cost) is weighted by the probability of the next action and states, i.e., $Q(U_{\tau+1},S_{\tau+1}|U_\tau,S_\tau)$. Intuitively, the weighting terms in SI discounts the impact of the back-propagated cost for unlikely states and (locally) sub-optimal actions. Importantly, those weighting terms emerge from the recursive definition of the EFE that relates to the Bellman equation \citep{dacosta2020relationship}. In contrast, there are no such weights in BTAI because BTAI finds its inspiration in active inference.

\section{Conclusion and future works} \label{sec:conclusion}

In this paper, we have presented a new approach where planning is cast as structure learning. Simply put, this approach consists of dynamically expanding branches of the generative model by evaluating alternative futures under different action sequences. The dynamic expansion trades off evaluating promising (with repect to the target distribution) policies with exploring policies whose outcomes are uncertain. We proposed two different tree search methods: the first in which the nodes' cost is propagated forward from the root node to the leaves; the second in which the nodes' cost is propagated backward from the leaves to the root node. Then, in Section \ref{sec:aits_ai_sai} we showed that forward propagation of the EFE leads to active inference (AcI) under the assumption that the probability of observations and states are independent of future actions, and that backward propagation relates to sophisticated inference (SI). This clarifies the link between AcI and BTAI, and helps to understand the relationship between AcI and SI.

Importantly, by performing a complexity class analysis, we have shown that while Active Inference suffers from an exponential complexity class, our approach scales nicely (linearly) with the number of tree expansions, c.f., Section 3.4.2 of our companion paper \citep{AITS_PRACTICE}. Of course, the total number of possible expansions grows exponentially but as has been empirically shown in the reinforcement learning literature, even complex tasks, such as chess or the game of go, can be performed efficiently with MCTS \citep{Go,MuZero}, i.e., the efficiency of MCTS-like approaches relies on the ability to guide the expansion procedure using either powerful heuristics (like the EFE) or neural networks or both.

We also know that humans engage in counterfactual reasoning \citep{rafetseder2013counterfactual}, which, in our planning context, could involve the consideration and evaluation of alternative (non-selected) sequences of decisions. It may be that, because of the more exhaustive representation of possible trajectories, the classic active inference can more efficiently engage in counterfactual reasoning. In contrast, branching-time active inference would require these alternatives to be generated ``a fresh" for each counterfactual deliberation. In this sense, one might argue that there is a trade-off: branching-time active inference provides considerably more efficient planning to attain current goals, classic active inference provides a more exhaustive assessment of paths not taken.

Now that we have laid out the mathematics of BTAI, many directions of research could be investigated. One could for example obtain an intuitive understanding of the model's parameters through experimental study. At first it might be necessary to restrict oneself to agents without learning, i.e., inference only. This step should help answer questions such as: How does the number of expansions of the tree and the quality of the prior preferences impact the quality of planning? What is the best inference method (i.e., local or global inference) to use during planning?

Then, one could consider learning of the transition and likelihood matrices as well as the vector of initial states. This can be done in at least two ways. The first is to add Dirichlet priors over those matrices/vectors and the second would be to use neural networks as function approximators. The second option will lead to a deep active inference agent \citep{PixelBasedAI,DeepAI} equipped with tree search that could be directly compared to the method of \citet{DeepAIwithMCMC}. Including deep neural networks in the framework will also enable direct comparison with the deep reinforcement learning literature \citep{DBLP:journals/corr/abs-1801-01290,DeepRL,DDQN,lample2016playing,Go}. These comparisons will enable the impact of epistemic terms to be studied when the agent is composed of deep neural networks.

Another, very important direction for future research would be the creation of a biologically plausible implementation of BTAI. For example, using artificial neural networks to model the various mappings of the framework may provide a neural-based implementation of BTAI that is closer to biology. This would especially be the case, if the back-propagation algorithm frequently used for learning is replaced by the (more biologically plausible) generalized recirculation algorithm \citep{DBLP:journals/neco/OReilly96}. Another possible approach would be to use populations of neurons to encode the update equations of the framework, as was proposed by \citet{believe}.

Whatever technique is chosen for learning and inference, implementing MCTS in a biologically plausible way will be challenging. Indeed, MCTS requires a dynamic expansion of the search tree used to explore the space of possible policies. Each time an expansion is performed, the agent needs to store the associated variables such as: the number of visits, the aggregated expected free energy, and the posterior beliefs of the newly expanded node. Given the fast pace at which planning must be performed to be useful, slow mechanisms such as synaptic plasticity and neurogenesis are likely to be unsuitable for the task. A more plausible approach might rely upon a change of neuronal activation, which can occur within a few hundred milliseconds. One such approach uses a binding pool \citep{bowman2007simultaneous} and provides a notion of variable. In this framework, a variable is composed of two parts. First, a \textit{token} that can be intuitively understood as the variable's name, and second, a \textit{type} corresponding to the variable's value. The binding pool is then composed of neurons representing the fact that a variable's name is \textit{bound} (or set) to a specific value. A localist realisation of a binding pool could be implemented as a 2D array of neurons of size ``number of tokens" $\times$ ``number of types". However, such a representation is quite inefficient and a more compact (i.e. distributed) representation has been developed \citep{wyble2006neural}. If variables are complex data structures, such as those required by BTAI, they can be realised in a neural substrate.

Finally, in this paper, we focused on the UCT criterion for node selection because it is a standard choice in the reinforcement learning literature \citep{Go,MuZero}. However, it would be valuable to consider alternative criteria such as Thompson Sampling \citep{ThompsonOriginal,TutoThompson,AI_and_Thompson} or expected improvement \citep{EI_for_BO,ExpectedInprovement}. For example, Thompson Sampling has been shown to improve upon the standard UCT criterion when applied to MCTS \citep{Thompson_and_MCTS}, but requires additional modelling and we leave this for future research.


\acks{
We would like to thank the reviewers for their valuable feedback, which greatly improved the quality of the present paper. LD is supported by the Fonds National de la Recherche, Luxembourg (Project code: 13568875). This publication is based on work partially supported by the EPSRC Centre for Doctoral Training in Mathematics of Random Systems: Analysis, Modelling and Simulation (EP/S023925/1).}

\vskip 0.2in
\bibliography{references}

\appendix

\section*{Appendix A: Generalized inner and outer products}

\paragraph{Generalized outer products:}

Given $N$ vectors $V^i$, the generalized outer product returns an $N$ dimensional array $W$, whose element in position $(x_1,...,x_N)$ is given by $V^1_{x_1}\times ... \times V^N_{x_N}$, where $V^i_{x_j}$ is the $x_j$-th element of the $i$-th vector. In other words:
\begin{align}
W = \otimes \Big[ V^1, ..., V^N \Big] \Leftrightarrow W(x_1,...,x_N) &= V^1_{x_1}\times ... \times V^N_{x_N} \nonumber\\
&\forall x_j \in \{1, ..., |V^j|\} \,\, \forall j \in \{1, ..., N\},
\end{align}
where $|V^j|$ is the number of elements in $V^j$. Also, note that by definition $W$ is a $N$-tensor of size $|V^1| \times ... \times |V^N|$. Figure \ref{fig:outer-product} illustrates the generalized outer product for $N = 3$.

\begin{figure}[H]
	\begin{center}
	\includegraphics[scale=0.9]{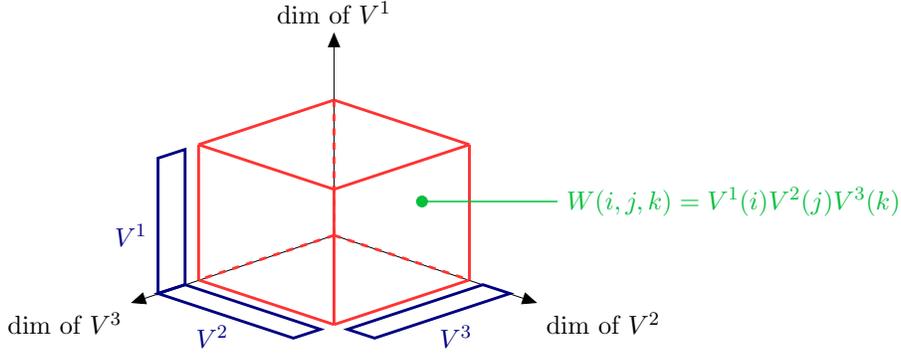}
	\end{center}
  \caption{This figure illustrates the generalized outer product $W = \otimes \big[ V^1, V^2, V^3 \big]$, where $W$ is a cube of values illustrated in red, whose typical element $W(i,j,k)$ is the product of $V^1(i)$, $V^2(j)$ and $V^3(k)$. Also, the vectors $V^i \,\, \forall i \in \{1, ..., 3\}$ are drawn in blue along the dimension of the cube they correspond to.}
   \label{fig:outer-product}
\end{figure}

\paragraph{Generalized inner products:}

Given an $N$-tensor $W$ and $M = N - 1$ vectors $V^i$, the generalized inner product returns a vector $Z$ obtained by performing a weighted average (with weighting coming from the vectors) over all but one dimension. In other words:
\begin{align}
Z = W \odot \Big[ V^1, ..., V^{M}\Big] \Leftrightarrow Z(x_j) = \sum_{\substack{x_1 \in \{1, ..., |V^1|\}\\
{\color{white}\}}...{\color{white}\}}\nonumber\\
x_M \in \{1, ..., |V^M|\} }} V^1_{x_1} \times ... \times W(x_1, ..., & x_j, ..., x_{M}) \times ... \times V^{M}_{x_{M}}\\
& \forall x_j \in \{1, ..., |Z|\},
\end{align}
where $|Z|$ denotes the number of elements in $Z$, and the large summand is over all $x_r$ for $r \in \{1, ..., M\} \setminus \{j\}$, i.e., excluding $j$. Also, note that if $|W|_{V^i} \,\, \forall i \in \{1, ..., M\}$ is the number of elements in the dimension corresponding to $V^i$, then for $W \odot \big[ V^1, ..., V^{M}\big]$ to be properly defined, we must have $|W|_{V^i} = |V^i| \,\, \forall i \in \{1, ..., M\}$ where $|V^i|$ is the number of elements in $V^i$. Figure \ref{fig:inner-product} illustrates the generalized inner product for $N = 3$.

\begin{figure}[H]
	\begin{center}
	\includegraphics[scale=0.8]{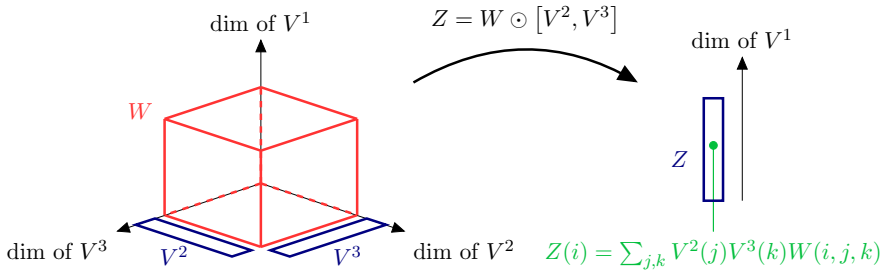}
	\end{center}
  \caption{This figure illustrates the generalized inner product $Z = W \odot \big[ V^2, V^3 \big]$, where $W$ is a cube of values illustrated in red with typical element $W(i,j,k)$. Also, the vectors $Z$ and $V^i \,\, \forall i \in \{2,3\}$ are drawn in blue along the dimension of the cube they correspond to.}
   \label{fig:inner-product}
\end{figure}

\paragraph{Naming of the dimensions:}

Importantly, we should imagine that each side of $W$ has a name, e.g., if $W$ is a 3x2 matrix, then the $i$-th dimension of $W$ could be named: ``the dimension of $V_i$". This enables us to write: $Z^1 = W \odot V^1$ and $Z^2 = W \odot V^2$, where $Z^1$ is a 1x2 matrix (i.e., a vector with two elements) and $Z^2$ is a 3x1 matrix (i.e., a vector with three elements). The operator $\odot$ knows (thanks to the dimension name) that $W \odot V^1$ takes the weighted average w.r.t ``the dimension of $V_1$", while $W \odot V^2$ must take the weighted average over ``the dimension of $V_2$".

In the context of active inference, the matrix $\bm{A}$ has two dimensions that we could call ``the observation dimension" (i.e., row-wise) and ``the state dimension" (i.e., column-wise). Trivially, $\bm{A} \odot \bm{o}_\tau$ will then correspond to the average of $\bm{A}$ along the observation dimension and $\bm{A} \odot \bm{\hat{D}}_\tau$ will correspond to the average of $\bm{A}$ along the state dimension. 

\section*{Appendix B: Generalized inner/outer products and other well-known products}

In this section, we explore the relationship between our generalized inner and outer products---presented in appendix A---and other well known products in the literature.

\subsection*{Inner product of two vectors}

The inner product of two vectors $\vec{a}$ and $\vec{b}$ of the same size is given by:
\begin{align}
\vec{a} \cdot \vec{b} = \sum_{i=1}^{|\vec{a}|} a_i b_i,
\end{align}
where $a_i$ and $b_i$ are the elements of the vectors $\vec{a}$ and $\vec{b}$, respectively, and $|\vec{a}|$ is the number of elements in $\vec{a}$. This product is a special case of our generalised inner product, i.e.,
\begin{align}
Z = W \odot V \Leftrightarrow Z = \sum_{i=1}^{|W|} W_i V_i,
\end{align}
where $Z$ is a scalar (a 0-tensor), $W$ and $V$ are two vectors (two 1-tensors) of the same size, and $|W|$ is the number of elements in $W$.

\subsection*{Inner product of two matrices (Frobenius inner product)}

The inner product of two matrices $\bm{A}$ and $\bm{B}$ of same sizes is given by:
\begin{align}
\langle \bm{A}, \bm{B}\rangle = \sum_{i=1}^{|\bm{A}|_1}\sum_{j=1}^{|\bm{A}|_2} a_{ij} b_{ij},
\end{align}
where $|\bm{A}|_1$ is the number of rows in $\bm{A}$, $|\bm{A}|_2$ is the number of columns in $\bm{A}$, and $a_{ij}$ ($b_{ij}$) are the elements of the matrices $\bm{A}$ (respectively $\bm{B}$). This product is not a special case of our generalised inner product.

\subsection*{Inner product of two tensors (Frobenius inner product)}

The inner product of two tensors $\bm{A}$ and $\bm{B}$ of same sizes is given by:
\begin{align}
\langle \bm{A}, \bm{B}\rangle = \sum_{i_1=1}^{|\bm{A}|_1} ... \sum_{i_n=1}^{|\bm{A}|_n} a(i_1, ..., i_n) b(i_1, ..., i_n),
\end{align}
where $a(i_1, ..., i_n)$ and $b(i_1, ..., i_n)$ are the elements of the tensors $\bm{A}$ and $\bm{B}$, respectively, and $|\bm{A}|_i$ is the number of elements in the i-th dimension of $\bm{A}$. This product is not a special case of our generalised inner product.

\subsection*{Standard matrix multiplication}

Let $\bm{A}$ be an $n \times m$ matrix and $\vec{b}$ be a vector of size $m$. The standard matrix multiplication of $\bm{A}$ by $\vec{b}$ is given by:
\begin{align}
\vec{c} = \bm{A} \vec{b} \Leftrightarrow \vec{c}_i = \sum_{j = 1}^{m} \bm{A}_{ij} \vec{b}_j.
\end{align}
This is a special case of our generalised inner product, i.e.,
\begin{align}
\vec{c} = \bm{A}\vec{b} = \bm{A} \odot \vec{b}.
\end{align}
Additionally, let $\bm{A}$ be an $n \times m$ matrix and $\vec{a}$ be a vector of size $n$, then:
\begin{align}
\vec{c} = \bm{A}^T \vec{a} \Leftrightarrow \vec{c}_i = \sum_{j = 1}^{n} \bm{A}_{ji} \vec{a}_j.
\end{align}
This is a special case of our generalised inner product, i.e.,
\begin{align}
\vec{c} = \bm{A}^T\vec{a} = \bm{A} \odot \vec{a}.
\end{align}
Note that because the dimensions are ``named" (c.f. Appendix A) the operator performs the transposition implicitly.

\subsection*{Outer product of two vectors}

Given two vectors $\vec{a}$ and $\vec{b}$, there outer product---denoted $\vec{a} \otimes \vec{b}$---is a matrix defined as:

\begin{align}
\vec{a} \otimes \vec{b} = \begin{bmatrix}
a_1b_1 & \ldots & a_1b_n\\
\vdots & \ddots & \vdots\\
a_mb_1 & \ldots & a_mb_n\\
\end{bmatrix}.
\end{align}

This outer product is a special case of our generalised outer product, where the operator is applied to only two vectors, i.e.
\begin{align}
\vec{a} \otimes \vec{b} = \otimes \Big[\vec{a}, \vec{b}\Big].
\end{align}

\subsection*{Outer product of two tensors}

Given two tensors $\bm{U}$ and $\bm{V}$, the outer product of $\bm{U}$ and $\bm{V}$ is another tensor $\bm{W}$ such that:
\begin{align}
\bm{W}(i_1,...,i_n,j_1,...,j_m) = \bm{V}(i_1,...,i_n)\bm{U}(j_1,...,j_m) \,\, &\forall i_\alpha \in \{1, ..., |\bm{V}|_\alpha \} \forall \alpha \in \{1, ..., n\}\nonumber\\
&\forall j_\beta \in \{1, ..., |\bm{U}|_\beta \} \forall \beta \in \{1, ..., m\}.
\end{align}
Given $N$ vectors $V^i \,\, \forall i \in \{1, ..., N\}$, our outer product is a sequence of outer tensor products, i.e.,
\begin{align}
\otimes \Big[ V^1, ..., V^N \Big] = \Bigg[\Big[V^1 \otimes_{\text{tensor}} V^2 \Big] \otimes_{\text{tensor}} ... \Bigg] \otimes_{\text{tensor}} V^N,
\end{align}
where $\otimes_{\text{tensor}}$ and $\otimes$ are the tensor and generalised outer products, respectively.

\subsection*{Kronecker product}

Given two matrices $\bm{A}$ and $\bm{B}$, the Kronecker product of $\bm{A}$ and $\bm{B}$---denoted $\otimes_K$---is a generalisation of the outer product from vectors to matrices defined as:
\begin{align}
\bm{A} \otimes_K \bm{B} = \begin{bmatrix}
a_{11}\bm{B} & \ldots & a_{1n}\bm{B}\\
\vdots & \ddots & \vdots\\
a_{m1}\bm{B} & \ldots & a_{mn}\bm{B}
\end{bmatrix},
\end{align}
where $a_{ij}$ are the elements of $\bm{A}$. Note that even if the Kronecker product is a generalisation of the outer product, it is neither a special case nor a generalisation of our generalized outer product.

\subsection*{Hadamard product}

Given two matrices $\bm{A}$ and $\bm{B}$ of the same size, the Hadamard product of $\bm{A}$ and $\bm{B}$ is an element-wise product defined by:
\begin{align}
\bm{C} = \bm{A} \odot \bm{B} \Leftrightarrow c_{ij} = a_{ij} b_{ij} \,\,\forall i \in \{1, ..., |\bm{A}|_1\} \forall j \in \{1, ..., |\bm{A}|_2\},
\end{align}
where $a_{ij}$, $b_{ij}$ and $c_{ij}$ are the elements in the $i$-th row and $j$-th column of $\bm{A}$, $\bm{B}$ and $\bm{C}$, respectively, $|\bm{A}|_1$ is the number of rows in $\bm{A}$, and $|\bm{A}|_2$ is the number of columns in $\bm{A}$. This product is unrelated to both our generalised inner and outer products.

\section*{Appendix C: Instance of variational message passing}

This appendix provides a concrete instance of the method of Winn and Bishop discussed in Section \ref{sec:vmp}. The generative model is as follows:
\begin{align}
P(S,\bm{D}) = P(S|\bm{D})P(\bm{D})
\end{align}
where:
\begin{align}
P(S|\bm{D}) = \text{Cat}(\bm{D})\\
P(\bm{D}) = \text{Dir}(\bm{d}).
\end{align}
Additionally, the variational distribution is given by:
\begin{align}
Q(\bm{D}) = \text{Dir}(\bm{\hat{d}}),
\end{align}
which means that we assume that $S$ is an observed random variable. Let us start with the definition of the Dirichlet and categorical distributions written in the form of the exponential family:
\begin{align}\label{eq:D:0}
\ln P(\bm{D}) &= \underbrace{\begin{bmatrix}
\bm{d}_1 - 1\\
...\\
\bm{d}_{|S|} - 1
\end{bmatrix}}_{\mu_{D}(\bm{d})} \cdot
\underbrace{\begin{bmatrix}
\ln \bm{D}_1\\
...\\
\ln \bm{D}_{|S|}
\end{bmatrix}}_{u_{D}(\bm{D})} - 
\underbrace{\ln B(d)}_{z_{D}(\bm{d})}
\end{align}

\begin{align}\label{eq:D:1}
\ln P(S|\bm{D}) &= \underbrace{\begin{bmatrix}
\ln \bm{D}_1\\
...\\
\ln \bm{D}_{|S|}
\end{bmatrix}}_{\mu_{S}(\bm{D})}
\cdot
\underbrace{\begin{bmatrix}
[S = 1]\\
...\\
[S = |S|]
\end{bmatrix}}_{u_{S}(S)}
\end{align}
where $\cdot$ performs an inner product of the two vectors it is applied to, $B(\bm{d})$ is the Beta function and $|S|$ is the number of values a state can take. The first step requires us to re-write Equation \ref{eq:D:1} as a function of $u_{D}(\bm{D})$, which is straightforward because $\mu_{S}(\bm{D})$ is just another name for $u_{D}(\bm{D})$. Using the fact that the inner product is commutative:
\begin{align}\label{eq:D:2}
\ln P(S|\bm{D}) &= \underbrace{\begin{bmatrix}
[S = 1]\\
...\\
[S = |S|]
\end{bmatrix}}_{\mu_{S\rightarrow D}(S)}
\cdot
\underbrace{\begin{bmatrix}
\ln \bm{D}_1\\
...\\
\ln \bm{D}_{|S|}
\end{bmatrix}}_{u_{D}(\bm{D})}.
\end{align}

The second step aims to substitute \eqref{eq:D:0} and \eqref{eq:D:2} into the variational message passing equation \eqref{eq:5}, i.e.
\begin{align}
\ln Q^*(\bm{D}) &= \Big\langle \underbrace{\begin{bmatrix}
{\color{orange}\bm{d}_1 - 1}\\
...\\
{\color{orange}\bm{d}_{|S|} - 1}
\end{bmatrix}}_{\mu_{D}(\bm{d})} \cdot
\underbrace{\begin{bmatrix}
\ln \bm{D}_1\\
...\\
\ln \bm{D}_{|S|}
\end{bmatrix}}_{u_{D}(\bm{D})} - 
\underbrace{\ln B(\bm{d})}_{z_{D}(\bm{d})}
\Big\rangle + \Big\langle \underbrace{\begin{bmatrix}
{\color{purple}[S = 1]}\\
...\\
{\color{purple}[S = |S|]}
\end{bmatrix}}_{\mu_{S\rightarrow D}(S)}
\cdot
\underbrace{\begin{bmatrix}
\ln \bm{D}_1\\
...\\
\ln \bm{D}_{|S|}
\end{bmatrix}}_{u_{D}(\bm{D})} \Big\rangle + \text{Const},
\end{align}
where $\langle \bigcdot \rangle$ refers to $\langle \bigcdot \rangle_{\sim Q_{\bm{D}}}$. Note that in the above equation, $\bm{d}_i$ are fixed parameters, therefore there is no posterior over $d$ and the first expectation $\langle \cdot \rangle_{\sim Q_{\bm{D}}}$ can be removed. The third step rests on taking the exponential of both sides, using the linearity of expectation and factorising by $u_{D}(\bm{D})$ to obtain:
\begin{align}\label{eq:D:3.0}
Q^*(\bm{D}) &= \exp \Bigg\{ \begin{bmatrix}
{\color{orange}\bm{d}_1 - 1} + \langle {\color{purple}[S = 1]} \rangle\\
...\\
{\color{orange}\bm{d}_{|S|} - 1} + \langle {\color{purple}[S = |S|]} \rangle
\end{bmatrix}
\cdot u_{D}(\bm{D}) + \text{Const} \Bigg\},
\end{align}
where $z_{D}(\bm{d})$ have been absorbed into the constant term because it does not depend on $\bm{D}$. The fourth step is a re-parameterisation done by observing that $\langle [S = i] \rangle$ is the i-th element of the expectation of the vector $u_{S}(S)$, i.e. $\langle u_{S}(S) \rangle_i = \langle [S = i] \rangle$:

\begin{align}\label{eq:D:3}
Q^*(\bm{D}) &= \exp \Bigg\{ \underbrace{\begin{bmatrix}
{\color{orange}\bm{d}_1 - 1} + \langle {\color{purple}u_{S}(S)} \rangle_1\\
...\\
{\color{orange}\bm{d}_{|S|} - 1} + \langle {\color{purple}u_{S}(S)} \rangle_{|S|}
\end{bmatrix}}_{\tilde{\mu}_{D}(...) + \tilde{\mu}_{S \rightarrow D}(...)}
\cdot u_{D}(\bm{D}) + \text{Const} \Bigg\}.
\end{align}

The last step consists of computing the expectation of $\langle u_{S}(S) \rangle_i$ for all $i$. This can be achieved by realising that the probability of an indicator function for an event is the probability of this event, i.e $\langle u_{S}(S) \rangle_i = \langle [S = i] \rangle = Q(S = i) = \hat{\bm{D}}_{i}$ where $\hat{\bm{D}}$ is a one hot representation of the observed value for $S$. Substituting this result in Equation \ref{eq:D:3}, leads to the final result:
\begin{align}
Q^*(\bm{D}) &= \exp \Bigg\{ \begin{bmatrix}
{\color{orange}\bm{d}_1 - 1} + {\color{purple}\hat{\bm{D}}_{1}}\\
...\\
{\color{orange}\bm{d}_{|S|} - 1} + {\color{purple}\hat{\bm{D}}_{|S|}}
\end{bmatrix}
\cdot u_{D}(\bm{D}) + \text{Const}\Bigg\}.
\end{align}

Indeed, the above equation is in fact a Dirichlet distribution in exponential family form, and can be re-written into its usual form to obtain the final update equation:
\begin{align}
Q^*(\bm{D}) &= \text{Dir}({\color{orange}\bm{d}} + {\color{purple}\hat{\bm{D}}}).
\end{align}

\section*{Appendix D: Expected log of Dirichlet distribution}

\begin{definition}
A probability distribution over $\bm{x}$ parameterized by $\bm{\mu}$ is said to belong to the exponential family if its probability mass function $P(\bm{x}|\bm{\mu})$ can be written as:
\begin{align}
P(\bm{x}|\bm{\mu}) = h(\bm{x}) \exp\Big[ \bm{\mu} \cdot T(\bm{x}) - A(\bm{\mu})\Big],
\end{align}
where $h(\bm{x})$ is the base measure, $\bm{\mu}$ is the vector of natural parameters, $T(\bm{x})$ is the vector of sufficient statistics, and $A(\bm{\mu})$ is the log partition.
\end{definition}

\begin{lemma}\label{lemma:log_partition}
The log partition is given by:
\begin{align}
A(\bm{\mu}) = \ln \int h(\bm{x}) \exp \Big[\bm{\mu} \cdot T(\bm{x})\Big] d\bm{x}.
\end{align}
\end{lemma}

\begin{proof}
Starting with the fact that $P(\bm{x}|\bm{\mu})$ integrate to one:
\begin{align}
\int P(\bm{x}|\bm{\mu}) d\bm{x} = &\int h(\bm{x})\exp \Big[\bm{\mu} \cdot T(\bm{x}) - A(\bm{\mu}) \Big] d\bm{x} = 1\\
\Leftrightarrow \quad &\frac{1}{\exp A(\bm{\mu})} \int h(\bm{x})\exp \Big[\bm{\mu} \cdot T(\bm{x}) \Big] d\bm{x} = 1\\
\Leftrightarrow \quad &A(\bm{\mu}) = \ln \int h(\bm{x})\exp \Big[\bm{\mu} \cdot T(\bm{x}) \Big] d\bm{x}
\end{align}
\end{proof}

\begin{lemma}\label{lemma:expectation_gradient}
The gradient of the log partition function is the expectation of the sufficient statistics, i.e.,
\begin{align}
\frac{\partial A(\bm{\mu})}{\partial \bm{\mu}} = \mathbb{E}_{P(\bm{x}|\bm{\mu})}\Big[ T(\bm{x})\Big].
\end{align}
\end{lemma}

\begin{proof}
Restarting with the derivative of the result of Lemma \ref{lemma:log_partition}:
\begin{align}
\frac{\partial A(\bm{\mu})}{\partial \bm{\mu}} = \frac{\partial}{\partial \bm{\mu}} \Bigg[ \ln \int h(\bm{x})\exp \Big[\bm{\mu} \cdot T(\bm{x}) \Big] d\bm{x} \Bigg],
\end{align}
and using the chain rule:
\begin{align}
\frac{\partial A(\bm{\mu})}{\partial \bm{\mu}} = \frac{1}{\int h(\bm{x})\exp\big[\bm{\mu} \cdot T(\bm{x}) \big] d\bm{x}} \frac{\partial}{\partial \bm{\mu}} \Bigg[ \int h(\bm{x})\exp \Big[\bm{\mu} \cdot T(\bm{x}) \Big] d\bm{x} \Bigg].
\end{align}
Note that the denominator of the first term is equal to the exponential of $A(\bm{\mu})$, and we can swap the derivative and the integral because the limit of integration does not depend on the parameters $\bm{\mu}$:
\begin{align}
\frac{\partial A(\bm{\mu})}{\partial \bm{\mu}} &= \frac{1}{\exp A(\bm{\mu})} \int \frac{\partial}{\partial \bm{\mu}} \bigg[ h(\bm{x})\exp \Big[\bm{\mu} \cdot T(\bm{x}) \Big] \bigg] d\bm{x}\\
&= \frac{1}{\exp A(\bm{\mu})} \int h(\bm{x}) \frac{\partial}{\partial \bm{\mu}} \bigg[ \exp \Big[\bm{\mu} \cdot T(\bm{x}) \Big] \bigg] d\bm{x}.
\end{align}
Using the chain rule again:
\begin{align}
\frac{\partial A(\bm{\mu})}{\partial \bm{\mu}} &= \frac{1}{\exp A(\bm{\mu})} \int h(\bm{x}) \exp \Big[\bm{\mu} \cdot T(\bm{x}) \Big] T(\bm{x}) d\bm{x}\\
&= \int h(\bm{x}) \exp \Big[\bm{\mu} \cdot T(\bm{x}) - A(\bm{\mu}) \Big] T(\bm{x}) d\bm{x}\\
&= \int P(\bm{x}|\bm{\mu}) T(\bm{x}) d\bm{x}\\
&= \mathbb{E}_{P(\bm{x}|\bm{\mu})} \Big[ T(\bm{x}) \Big].
\end{align}
\end{proof}

\begin{theorem}
If $\bm{D}$ is distributed according to a Dirichlet distribution $Q(\bm{D}) = \text{Dir}(\bm{D};\bm{\hat{d}})$, then:
$\bm{\mathring{D}} = \mathbb{E}_{Q(\bm{D})}\big[ \ln \bm{D} \big] \Leftrightarrow \bm{\mathring{D}}_i = \psi(d_i) - \psi(\sum_j d_j).$
\end{theorem}

\begin{proof}
\noindent Let $\bm{\mu}$ be equal to $\bm{\hat{d}} - \vec{1}$. Taking the exponential of both sides in Equation \ref{eq:D:0} and using that $\bm{\hat{d}} = \bm{\mu} + \vec{1}$, we obtain:
\begin{align}
Q(\bm{D}) &= \exp \Bigg\{ \underbrace{\begin{bmatrix}
\bm{\hat{d}}_1 - 1\\
...\\
\bm{\hat{d}}_{|S|} - 1
\end{bmatrix}}_{\bm{\mu}} \cdot
\underbrace{\begin{bmatrix}
\ln \bm{D}_1\\
...\\
\ln \bm{D}_{|S|}
\end{bmatrix}}_{T(\bm{D})} - 
\underbrace{\ln B(\bm{\mu} + \vec{1})}_{A(\bm{\mu})}\Bigg\},
\end{align}
where $\bm{\mu}$ is the vector of natural parameters, $T(\bm{D})$ is the vector of sufficient statistics, $A(\bm{\mu})$ is the log partition, and $B(\bigcdot)$ is the beta function. Using the result of Lemma \ref{lemma:expectation_gradient}:
\begin{align}
\bm{\mathring{D}} \delequal \mathbb{E}_{Q(\bm{D})}\Big[ T(\bm{D})\Big] = \mathbb{E}_{Q(\bm{D})}\Big[ \ln \bm{D} \Big] = \frac{\partial A(\bm{\mu})}{\partial \bm{\mu}}= \frac{\partial}{\partial \bm{\mu}} \Big[ \ln B(\bm{\mu} + \vec{1})\Big].
\end{align}
We now focus on a typical element of $\bm{\mathring{D}}$:
\begin{align}
\bm{\mathring{D}}_i = \frac{\partial}{\partial \mu_i}\Big[ \ln B(\bm{\mu} + \vec{1})\Big],
\end{align}
and use the definition of the beta function:
\begin{align}
\bm{\mathring{D}}_i &= \frac{\partial}{\partial \bm{\mu}_i}\bigg[\ln \frac{\prod_k \Gamma(\bm{\mu}_k + 1)}{\Gamma(\sum_k \bm{\mu}_k + 1)}\bigg]\\
&= \frac{\partial}{\partial \bm{\mu}_i}\bigg[\sum_k \ln \Gamma(\bm{\mu}_k + 1) - \ln \Gamma({\textstyle \sum_k \bm{\mu}_k + 1})\bigg]\\
&= \frac{\partial}{\partial \bm{\mu}_i}\bigg[ \ln \Gamma(\bm{\mu}_i + 1)\bigg] - \frac{\partial}{\partial \bm{\mu}_i}\bigg[\ln \Gamma({\textstyle \sum_k \bm{\mu}_k + 1})\bigg],
\end{align}
where $\Gamma(\bigcdot)$ is the gamma function. The last step relies on the chain rule:
\begin{align}
\bm{\mathring{D}}_i &= \frac{\partial}{\partial (\bm{\mu}_i + 1)}\bigg[ \ln \Gamma(\bm{\mu}_i + 1)\bigg] \underbrace{\frac{\partial \bm{\mu}_i + 1}{\partial \bm{\mu}_i}}_{=1} - \frac{\partial}{\partial ({\textstyle \sum_k \bm{\mu}_k + 1})}\bigg[\ln \Gamma({\textstyle \sum_k \bm{\mu}_k + 1})\bigg]\underbrace{\frac{\partial {\textstyle \sum_k \bm{\mu}_k + 1}}{\partial \bm{\mu}_i}}_{=1}\\
&= \frac{\partial}{\partial (\bm{\mu}_i + 1)}\bigg[ \ln \Gamma(\bm{\mu}_i + 1)\bigg] - \frac{\partial}{\partial ({\textstyle \sum_k \bm{\mu}_k + 1})}\bigg[\ln \Gamma({\textstyle \sum_k \bm{\mu}_k + 1})\bigg]\\
&= \psi(\bm{\mu}_i + 1) - \psi({\textstyle \sum_k \bm{\mu}_k + 1})\\
&= \psi(\bm{\hat{d}}_i) - \psi({\textstyle \sum_k \bm{\hat{d}}_k}),
\end{align}
where we used that $\bm{\hat{d}} = \bm{\mu} + \vec{1}$ and the definition of the digamma function, i.e., $\psi(x) = \frac{\partial \ln \Gamma(x)}{\partial x}$.
\end{proof}

\section*{Appendix E: Relationship between BTAI and active inference (Lemmas)}

\begin{lemma}\label{lemma_51}
Under the assumption that the probability of observations and states are independent of future actions, i.e.,
\begin{align}
\forall j \in \mathbb{N}_{>0}, \,\, Q(O_{t+i}|\pi_{i}) \approx Q(O_{t+i}|\pi_{i+j}) \text{ and } Q(S_{t+i}|\pi_{i}) \approx Q(S_{t+i}|\pi_{i+j}),
\end{align}
then:
\begin{align}
\forall j \in \mathbb{N}_{>0},\,\, \bm{G}(\pi_{i+j}, t+i) \approx \bm{G}(\pi_i, t+i).
\end{align}
\end{lemma}

\begin{proof}
The proof is straightforward, we start with the following definition:
\begin{align}
\bm{G}(\pi_{i+j}, t+i) = D_{\mathrm{KL}}[Q(O_{t+i}|\pi_{i+j})||P(O_{t+i})]\,\, +\,\, \mathbb{E}_{Q(S_{t+i}|\pi_{i+j})}[\text{H}[P(O_{t+i}| S_{t+i})]].
\end{align}
Then, using the assumption that the probability of observations and states are independent of future actions, i.e., $\forall j \in \mathbb{N}_{>0},\,\, Q(O_{t+i}|\pi_{i+j}) \approx Q(O_{t+i}|\pi_{i})$ and $\forall j \in \mathbb{N}_{>0},\,\, Q(S_{t+i}|\pi_{i+j}) \approx Q(S_{t+i}|\pi_i)$, we get:
\begin{align}
\bm{G}(\pi_{i+j}, t+i) \approx D_{\mathrm{KL}}[Q(O_{t+i}|\pi_i)||P(O_{t+i})]\,\, +\,\, \mathbb{E}_{Q(S_{t+i}|\pi_i)}[\text{H}[P(O_{t+i}| S_{t+i})]] \delequal \bm{G}(\pi_i, t+i).
\end{align}
\end{proof}

\begin{lemma}\label{lemma_52}
The aggregated cost for an arbitrary $N$ is given by:
\begin{align}
\bm{G}^{aggre}_{\pi_N} = \sum_{i=1}^{N} \bm{G}(\pi_i, t + i),
\end{align}
\end{lemma}

\begin{proof}
The proof is done by induction. The initialisation holds for $N = 1$, indeed, $\pi_{1} = \{ U_t \}$ and by definition:
\begin{align}
\underbrace{\bm{G}(\pi_{1})}_{\bm{G}_{\pi_{1}}^{aggre}} = &\sum_{\tau = t+1}^{t+1} \bm{G}(\pi_{1}, \tau) = \underbrace{D_{\mathrm{KL}}[Q(O_{t+1}|\pi_{1})||P(O_{t+1})]\,\, +\,\, \mathbb{E}_{Q(S_{t+1}|\pi_{1})}[\text{H}[P(O_{t+1} | S_{t+1})]]}_{\bm{G}_{\pi_{1}}^{local}},\\
&\Leftrightarrow \bm{G}_{\pi_{1}}^{aggre} = \bm{G}_{\pi_{0}}^{aggre} + \bm{G}_{\pi_{1}}^{local},
\end{align}
because by definition $\bm{G}_{\pi_{0}}^{aggre} = 0$. Then, assuming that $\bm{G}^{aggre}_{\pi_N} = \sum_{i=1}^{N} \bm{G}(\pi_i, t + i)$ holds for some $N$, we show that its hold for $N+1$ as well. By definition:
\begin{align}
\bm{G}^{aggre}_{\pi_{N+1}} = \bm{G}^{aggre}_{\pi_N} + \bm{G}^{local}_{\pi_{N+1}},
\end{align}
and:
\begin{align}
\bm{G}^{local}_{\pi_{N+1}} = \bm{G}(\pi_{N+1},t+N+1).
\end{align}
Using the inductive hypothesis and the above two equations:
\begin{align}
\bm{G}^{aggre}_{\pi_{N+1}} &= \sum_{i=1}^{N} \bm{G}(\pi_i, t + i) + \bm{G}(\pi_{N+1},t+N+1)\\
&= \sum_{i=1}^{N+1} \bm{G}(\pi_i, t + i).
\end{align}
\end{proof}

\section*{Appendix F: Notation}

In this appendix, we introduce the notation used throughout this paper. The following sub-sections describe the notation related to sets of numbers, tensors, probability distributions, global variables, multi-indices and random variables, respectively.

\subsection*{Sets of numbers}

\begin{definition}
Let $\mathbb{N}_{>0}$ be the set of all strictly positive integers defined as:
\begin{align}
\mathbb{N}_{>0} = \{ x \in \mathbb{N} \mid x > 0 \},
\end{align}
where $\mathbb{N}$ is the set of all natural numbers.
\end{definition}

\begin{definition}
Let $\mathbb{R}_{>0}$ be the set of all strictly positive real numbers defined as:
\begin{align}
\mathbb{R}_{>0} = \{ x \in \mathbb{R} \mid x > 0 \},
\end{align}
where $\mathbb{R}$ is the set of all real numbers.
\end{definition}

\subsection*{Tensors}

\begin{definition}
An $\bm{n}$\textbf{-tensor} is an $n$-dimensional array of values. Each element of an $n$-tensor is indexed by an $n$-tuple of non-negative integers, i.e., $(x_1, ..., x_n)$ where $x_i \in \mathbb{N}_{>0} \,\, \forall i \in \{1,2,...,n\}$.
\end{definition}

\begin{definition}
Let $T$ be an $n$-tensor. The \textbf{element} of $T$ indexed by the $n$-tuple $(x_1, ..., x_n)$ is a real number denoted by $T(x_1, ..., x_n)$.
\end{definition}

\begin{example}
Let $T$ be a 2-tensor defined as:
\begin{align}
T = \begin{bmatrix}
1 & 2\\
3 & 4
\end{bmatrix}.
\end{align}
Then $T(1,1) = 1$, $T(1,2) = 2$, $T(2,1) = 3$ and $T(2,2) = 4$.
\end{example}

\begin{remark}
A $0$-tensor is a scalar, a $1$-tensor is a vector, and a $2$-tensor is a matrix.
\end{remark}

\begin{definition}
Let $T$ be an $n$-tensor. The \textbf{size} of $T$ is a vector of size $n$ denoted $|T|$ whose $i$-th element corresponds to the size of the $i$-th dimension of $T$.
\end{definition}

\begin{example}
Let $T$ be a 2-tensor defined as:
\begin{align}
T = \begin{bmatrix}
1 & 2 & 3\\
4 & 5 & 6
\end{bmatrix}.
\end{align}
Then $|T|_1 = 2$ and $|T|_2 = 3$.
\end{example}

\begin{definition}
Let $T$ be an $n$-tensor. A \textbf{1-sub-tensor} of $T$ is a 1-tensor obtained by selecting a 1-dimensional slice of $T$, i.e., 
\begin{align}
T(x_1, ..., x_{i-1}, \bigcdot, x_{i+1}, ,..., x_n),
\end{align}
where $\bigcdot$ represents the selection of a $1$-dimensional slice of $T$, and the values of all $x_{j \neq i}$ must be set to specific values in $\{1, ..., |T|_j\}$. Figure \ref{fig:tensor-slices} (left) illustrates the notion of a 1-dimensional slice.
\end{definition}

\begin{figure}[H]
	\begin{center}
	\includegraphics[scale=0.8]{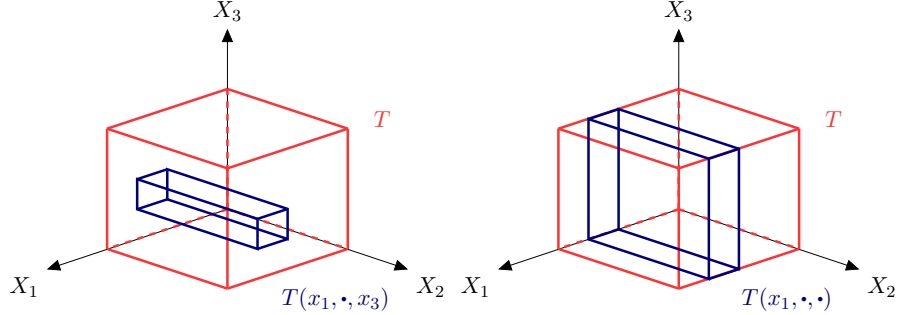}
	\end{center}
  \caption{This figure illustrates the notion of a $1$-dimensional slice (on the left) and of a $2$-dimensional slice (on the right).}
   \label{fig:tensor-slices}
\end{figure}

\begin{definition}
Let $T$ be an $n$-tensor and $m < n$. An $\bm{m}$\textbf{-sub-tensor} of $T$ (denoted $W$) is an $m$-tensor obtained by selecting an $m$-dimensional slice of $T$, i.e., 
\begin{align}
W = T(x_1, ..., x_{i_1-1}, \bigcdot, x_{i_1+1},......, x_{i_m-1}, \bigcdot, x_{i_m+1},..., x_n),
\end{align}
where $i_k \in \{1, ..., n\} \,\, \forall k \in \{1, ..., m\}$ are indices representing the dimension being selected. Naturally, the $k$-th dimension of $W$ corresponds to the $i_k$-th dimension of $T$ for $k \in \{ 1, ..., m\}$. Importantly, the sextuple of dots in the middle of the expression represents that there will be $m$ symbols $``\bigcdot"$, i.e., one for each dimension selected. Figure \ref{fig:tensor-slices} (right) illustrates the special case of a $2$-dimensional slice, i.e., $m = 2$.
\end{definition}

\begin{example}
Let $T$ be a 3-tensor such that:
\begin{itemize}
\item $|T|_1 = 2$
\item $T(1,x_2,x_3) = 1, \,\, \forall (x_2,x_3) \in \{1, ..., |T|_2\} \times \{1, ..., |T|_3\}$
\item $T(2,x_2,x_3) = 2, \,\, \forall (x_2,x_3) \in \{1, ..., |T|_2\} \times \{1, ..., |T|_3\}$.
\end{itemize}
Then $T(1,\bigcdot,\bigcdot)$ is a 2-sub-tensor of $T$ full of ones, and $T(2,\bigcdot,\bigcdot)$ is a 2-sub-tensor of $T$ full of twos.
\end{example}

\subsection*{Probability distributions}

\begin{definition}
A \textbf{random} $\bm{n}$\textbf{-tensor} is an $n$-tensor over which we have an n-dimensional probability distribution.
\end{definition}

\begin{remark}
A random variable is a random 0-tensor, a random vector is a random 1-tensor, and a random matrix is a random 2-tensor.
\end{remark}

\begin{definition}\label{def:tensor_describes_joint_dist}
An $n$-tensor $T$ is said to \textbf{represent a joint distribution} over a set of $n$ random variables $\{X_1, ..., X_n\}$ if:
\begin{align}
P(X_1 = x_1, ..., X_n = x_n) = T(x_1, ..., x_n).
\end{align}
For conciseness, if $T$ represents $P(X_1, ..., X_n)$ we let:
\begin{align}
P(X_1, ..., X_n) = \text{Cat}(T).
\end{align}
\end{definition}

\begin{remark}
If $T$ represents $P(X_1, ..., X_n)$, then the sum of its elements must equal one.
\end{remark}

\begin{remark}
In contrast, if $T$ is a \textbf{random} $n$-tensor, then:
\begin{align}
P(X_1, ..., X_n|T) = \text{Cat}(T),
\end{align}
which means that the joint probability over $X_1, ..., X_n$ is represented by $T$, and because $T$ is a \textbf{random} tensor (taking values in the set of valid $n$-tensors $\mathbb{T}_n$, i.e., the set of all $n$-tensors whose elements sum up to one), we must specify which instance of $T \in \mathbb{T}_n$ should be used to define the joint distribution over $X_1, ..., X_n$. 
\end{remark}

\begin{definition}\label{def:tensor_describes_cond_dist}
An $m$-tensor $R$ is said to \textbf{represent a conditional distribution} over a set of $m$ random variables $\{X_1, ..., X_m\}$ if:
\begin{align}
P(X_1 = x_1, ..., X_n = x_n | X_{n+1} = x_{n+1}, ..., X_m = x_m) = R(x_1, ..., x_m).
\end{align}
For conciseness, if $R$ represents $P(X_1, ..., X_n | X_{n+1}, ..., X_m)$ we let:
\begin{align}
P(X_1, ..., X_n | X_{n+1}, ..., X_m) = Cat(R).
\end{align}
\end{definition}

\begin{remark}
If $R$ represents $P(X_1, ..., X_n | X_{n+1}, ..., X_m)$, then the elements of the $n$-sub-tensor $R(\bigcdot, ..., \bigcdot, x_{n+1}, ..., x_m)$ must sum to one $\forall x_i \in \{1, ..., |R|_i\}\,\, \forall i \in \{n+1, ..., m\}$.
\end{remark}

\begin{remark}
If $R$ is a random $m$-tensor, then:
\begin{align}
P(X_1, ..., X_n | X_{n+1}, ..., X_m, R) = Cat(R).
\end{align}
\end{remark}

\begin{remark}
Importantly, definition \ref{def:tensor_describes_cond_dist} uses the symbol $T$ to represent $P(X_1, ..., X_n)$ and definition \ref{def:tensor_describes_joint_dist} uses the symbol $R$ to represent 
$P(X_1, ..., X_n | X_{n+1}, ..., X_m)$. Throughout this document, different symbols will be used for representing joint and conditional distributions.
\end{remark}

\begin{definition} \label{def:dirichlet_prior}
Let $R$ be a random $m$-tensor representing $P(X_1, ..., X_n | X_{n+1}, ..., X_m, R)$ and $k = m - n$ be the number of variables upon which the variables $X_1, ..., X_n$ are conditioned. Having a \textbf{Dirichlet prior} over $R$ means that:
\begin{align}
P(R) = \prod_{i_1 = 1}^{|R|_{n+1}} ... \prod_{i_k = 1}^{|R|_m} Dir\Big(\bm{r}(i_1, ..., i_k, \bigcdot ) \Big),
\end{align}
where $\bm{r}$ is an $(m + 1)$-tensor such that the $1$-sub-tensor $\bm{r}(i_1, ..., i_k, \bigcdot)$ contains the parameters of the Dirichlet prior over $P(X_1, ..., X_n | X_{n+1} = i_1, ..., X_m = i_k)$. For conciseness, we denote the prior over $R$ as:
\begin{align}
P(R) = Dir(\bm{r}).
\end{align}
\end{definition}

\begin{remark}
Definition \ref{def:dirichlet_prior} implicitly means that if $V$ is a 1-tensor then $Dir(V)$ represents a Dirichlet distribution. However, if $V$ is an $m$-tensor (with $m \neq 1$) then $Dir(V)$ represents a product of Dirichlet distributions.
\end{remark}

\begin{remark}
If $R$ is a random $m$-tensor representing $P(X_1, ..., X_n | X_{n+1}, ..., X_m, R)$, then its prior will be a product of $|X_{n+1}| \times ... \times |X_m|$ Dirichlet distributions, where $|X_i|$ is the number of values that $X_i$ can take. Additionally, each Dirichlet distribution will have $|X_1| \times ... \times |X_n|$ parameters stored in the last dimension of $\bm{r}$, where $\bm{r}$ is the tensor storing the parameters of the prior over $R$, i.e. $P(R) = Dir(\bm{r})$.
\end{remark}

\begin{example}
Let $\bm{A}$ be a random 2-tensor representing $P(O|S,\bm{A})$, then the Dirichlet prior over $\bm{A}$ is given by:
\begin{align}
P(\bm{A}) = Dir(\bm{a}) \delequal \prod_{i = 1}^{|\bm{A}|_2} Dir\Big(\bm{a}(i, \bigcdot)\Big),
\end{align}
where $|\bm{A}|_2$ is the number of values that $S$ can take (i.e., the number of hidden states).
\end{example}

\begin{example}
Let $\bm{B}$ be a random 3-tensor representing $P(S_{\tau + 1}|S_\tau, U_\tau, \bm{B})$, then the Dirichlet prior over $\bm{B}$ is given by:
\begin{align}
P(\bm{B}) = Dir(\bm{b}) \delequal \prod_{i_1 = 1}^{|\bm{B}|_2} \prod_{i_2 = 1}^{|\bm{B}|_3} Dir\Big(\bm{b}(i_1, i_2, \bigcdot)\Big),
\end{align}
where $|\bm{B}|_2$ is the number of values that $S_\tau$ can take (i.e., the number of hidden states) and $|\bm{B}|_3$ is the number of values that $U_\tau$ can take (i.e., the number of actions).
\end{example}

\subsection*{Global labels}

\begin{definition}
The \textbf{number of actions} available to the agent is denoted $|U|$.
\end{definition}

\begin{definition}
The \textbf{number of states} in the environment is denoted $|S|$.
\end{definition}

\begin{definition}
The \textbf{number of observations} that the agent can make is denoted $|O|$.
\end{definition}

\begin{definition}
The \textbf{number of policies} that the agent can pick from is denoted $|\pi|$.
\end{definition}

\begin{definition}
The time point representing the \textbf{present} is a natural number denoted $t$.
\end{definition}

\begin{definition}
The \textbf{time-horizon} (i.e., the time point after which the agent stops modelling the sequence of hidden states) is a natural number denoted $T$.
\end{definition}

\subsection*{Multi-indices}

\begin{definition}
A \textbf{multi-index} is a sequence of indices denoted by:
\begin{align}
I = (i_1, ..., i_n),
\end{align}
where $i_j \in \mathcal{D} \,\,\forall j \in \{1, ..., n\}$, and in this paper $\mathcal{D} = \{1, ..., |U|\}$.
\end{definition}

\begin{remark}
Multi-indices are used to index random variables such that $S_I$ is the hidden state obtained after taking the sequence of actions described by $I$, and $O_I$ is the random variable representing the observation generated by $S_I$.
\end{remark}

\begin{definition}
The \textbf{last index} of a multi-index is denoted $I_{last}$, i.e., $I_{last}$ is the last element of the sequence $I$.
\end{definition}

\begin{definition}
The \textbf{one-hot representation} of the action corresponding to $I_{last}$ is denoted $\vec{I}_{last}$.
\end{definition}

\begin{definition}
Given a multi-index $I$, $\IdMLast{I}$ corresponds to the sequence of actions described by $\bm{I}$ \textbf{without the last element}.
\end{definition}

\begin{remark}
In Section \ref{sec:ai_ts}, when a hidden state (i.e., $S_I$) is indexed by $I$, then $S_{\IdMLast{I}}$ will be the parent of $S_I$.
\end{remark}

\begin{definition}
Given an expandable generative model, $\mathbb{I}_t$ is the set of \textbf{all multi-indices already expanded} from the current state $S_t$.
\end{definition}

\begin{remark}
In Section \ref{sec:ai_ts} each time a hidden state (i.e., $S_I$) is added to the generative model, $I$ is added to the set of all multi-indices already expanded $\mathbb{I}_t$.
\end{remark}

\subsection*{Random variables and parameters of their distributions}

\begin{remark}
Parameters of the posterior distributions are recognizable by the hat notation, e.g., $\bm{\hat{a}}$, $\bm{\hat{b}}$ and $\bm{\hat{d}}$ will be posterior parameters, while $\bm{a}$, $\bm{b}$ and $\bm{d}$ will be prior parameters.
\end{remark}

\begin{remark}
The expected logarithm of an arbitrary tensor $\bm{X}$ representing a conditional or a joint distribution is denoted $\bm{\bar{X}}$, e.g. $\bm{\bar{A}} = \mathbb{E}_{Q(\bm{A})}[\ln \bm{A}]$ and $\bm{\bar{B}} = \mathbb{E}_{Q(\bm{B})}[\ln \bm{B}]$.
\end{remark}

\begin{definition}
Let $U_\tau$ be a random variable taking values in $\{1, ..., |U|\}$ indexing all possible actions. The prior distribution over $U_\tau$ is a categorical distribution represented by $\bm{\Theta}_\tau$. The posterior distribution over $U_\tau$ is a categorical distribution represented by $\bm{\hat{\Theta}}_\tau$.
\end{definition}

\begin{definition}
Let $S_\tau$ be a random variable taking values in $\{1, ..., |S|\}$ indexing all possible states. The prior and posterior distributions over $S_\tau$ are categorical distributions represented by different tensors depending on the generative model being considered. Therefore, those distributions are defined in Sections \ref{sec:ai} and \ref{sec:ai_ts}.
\end{definition}

\begin{definition}
Let $O_\tau$ be an observed random variable taking values in $\{1, ..., |O|\}$ indexing all possible observations. The prior distribution over $O_\tau$ is a categorical distribution conditioned on $S_\tau$ and represented by the random matrix $\bm{A}$. There is no posterior distribution over $O_\tau$ because $O_\tau$ is observed, i.e., realised.
\end{definition}

\begin{definition}
Let $O_I$ be a random variable taking values in $\{1, ..., |O|\}$ indexing all possible observations. The prior distribution over $O_I$ is a categorical distribution conditioned on $S_I$ and represented by the random matrix $\bm{\bar{A}}$. Note that $O_I$ refers to an observation in the future and is therefore a hidden variable. The posterior distribution over $O_I$ is a categorical distribution represented by $\bm{\hat{E}}_I$.
\end{definition}

\begin{definition}
Let $S_I$ be an random variable taking values in $\{1, ..., |S|\}$ indexing all possible states. The prior distribution over $S_I$ is a categorical distribution conditioned on $S_\IdMLast{I}$ and represented by the matrix $\bm{\bar{B}}_I \delequal \bm{\bar{B}}(\bigcdot, \bigcdot, I_{last})$. The posterior distribution over $S_I$ is a categorical distribution represented by $\bm{\hat{D}}_I$.
\end{definition}

\begin{definition}
Let $\bm{A}$ be a $|O| \times |S|$ random matrix defining the probability of an observation $O_\tau$ given the hidden state $S_\tau$. The prior distribution over $\bm{A}$ is a product of Dirichlet distributions whose parameters are stored in a $|S| \times |O|$ matrix $\bm{a}$. The posterior distributions over $\bm{A}$ is also a product of Dirichlet distribution, but the parameters are stored in a $|S| \times |O|$ matrix $\bm{\hat{a}}$.
\end{definition}

\begin{definition}
Let $\bm{B}$ be a $|S| \times |S| \times |U|$ random 3-tensor defining the probability of transiting from $S_\tau$ to $S_{\tau+1}$ when taking action $U_\tau$. The prior distribution over $\bm{B}$ is a product of Dirichlet distributions whose parameters are stored in a $|S| \times |U| \times |S|$ 3-tensor $\bm{b}$. The posterior distribution over $\bm{B}$ is also a product of Dirichlet distributions, but the parameters are stored in a $|S| \times |U| \times |S|$ 3-tensor $\bm{\hat{b}}$.
\end{definition}

\begin{definition}
Let $\bm{D}$ be a random vector of size $|S|$ defining the probability of the initial state $S_0$. The prior distribution over $\bm{D}$ is a Dirichlet distribution whose parameters are stored in a vector $\bm{d}$ of size $|S|$. The posterior distribution over $\bm{D}$ is also a Dirichlet distribution, but the parameters are stored in a vector $\bm{\hat{d}}$ of size $|S|$.
\end{definition}

\begin{definition}
Let $\bm{\Theta}_\tau \,\, \forall \tau \in \{0, ..., t-1\}$ be a random vector of size $|U|$ defining the probability of the action $U_\tau$. The prior distribution over $\bm{\Theta}_\tau$ is a Dirichlet distribution whose parameters are stored in a vector $\bm{\theta}_\tau$ of size $|U|$. The posterior distribution over $\bm{\Theta}_\tau$ is also a Dirichlet distribution, but the parameters are stored in a vector $\bm{\hat{\theta}}_\tau$ of size $|U|$.
\end{definition}

\begin{definition}
Let $\gamma$ be a random variable taking values in $\mathbb{R}_{>0}$. The prior distribution over $\gamma$ is a gamma distribution with shape parameter $\bm{\alpha} = 1$ and rate parameter $\bm{\beta} \in \mathbb{R}_{>0}$. The posterior distribution over $\gamma$ is a gamma distribution with shape parameter $\bm{\hat{\alpha}} = 1$ and rate parameter $\bm{\hat{\beta}} \in \mathbb{R}_{>0}$.
\end{definition}

\begin{definition}
Let $\pi$ be a random variable taking values in $\{1, ..., |\pi|\}$ indexing all possible policies. The prior distribution over $\pi$ is a softmax function of the vector $\bm{G}$ multiplied by minus the precision $\gamma$. Note that $\bm{G}$ is a vector of size $|\pi|$ whose $i$-th element is the expected free energy of the $i$-th policy. The posterior distribution over $\pi$ is a categorical distribution whose parameters are stored in a vector $\bm{\hat{\pi}}$ of size $|\pi|$.
\end{definition}

\section*{Appendix G: Multi-armed bandit problem}

In the multi-armed bandit problem, the agent is prompted with $K$ actions (one for each bandit's arm). Pulling the $i$-th arm returns a reward sampled from the reward distribution $P_i(X)$ associated to this arm. Let $\mu_i$ be the mean of the $i$-th reward distribution and $T_i(n)$ be the number of times the $i$-th bandit has been selected after $n$ plays. To solve the bandit problem, one needs to come up with an \textit{allocation strategy} that selects the action that minimises the agent's regret defined as:
\begin{align}
R_n = \mu^*n  - \sum_{i=1}^K \mu_i \mathbb{E} [T_i(n)],
\end{align}
where $\mu^*$ is the average reward of the best action. Note that an upper bound of $\mathbb{E} [T_i(n)]$ is derived by first upper bounding $T_i(n)$, and then using: the Chernoff-Hoeffing bound, the Bernstein inequality and some properties of $p$-series, c.f., proof of Theorem 1 in \citet{Auer2002} for details. So, if we first assume that,
\begin{align}
UCB1_i = \underbrace{{\color{white}\sqrt{\frac{1}{1}}} \bar{X}_i{\color{white}\sqrt{\frac{1}{1}}}}_{\text{exploitation}} + \underbrace{\sqrt{\frac{2 \ln n}{n_i}}}_{\text{exploration}},
\end{align}
where $n_i$ is the number of times the $i$-th action has been selected, and $\bar{X}_i$ is the average reward received after taking the $i$-th action. Then, the main result of \citet{Auer2002} was to show that if an allocation strategy was using the UCB1 criterion to select the next action, the expected regret of this allocation strategy will grow at most logarithmically in the number of plays $n$, i.e., $\mathcal{O}(\ln n)$. In addition, since it is known that the expected regret of the (best) allocation strategy grows at least logarithmically in $n$ \citep{UBC1_nln_opt}, we say that the UCB1 criterion resolves the exploration / exploitation trade-off, i.e., the UCB1 criterion ensures that the expected regret grows as slowly as possible.

\section*{Appendix H: The exponential complexity class}

In this appendix, we precisely pinpoint the exponential complexity class that is addressed in this paper, but first, we introduce a multi-index notation. Multi-indices will help us to refer to hidden states in the future. Naturally enough the indexes inside the multi-indices will correspond to the actions the agent will have to perform to reach the hidden state, e.g., $S_{(123)}$ corresponds to a hidden state at time $t + 3$ obtained by performing action 1 at time $t$, 2 at time $t + 1$ and 3 at time $t + 2$. Using this notation, Figure \ref{fig:exp_class} depicts all the possible policies up to two time steps in the future and the associated hidden states. Importantly, Figure \ref{fig:exp_class} shows that the number of policies grows exponentially with the number of time steps for which the agent tries to plan. Therefore, the definition of the prior over the policies, i.e., $P(\pi|\gamma) = \sigma(-\gamma \bm{G})$, exhibits an exponential space and time complexity class because the agent needs to store and compute the $|U|^T$ parameters of $P(\pi|\gamma)$, where $T$ is the time-horizon.

\begin{figure}[H]
	\begin{center}
	\includegraphics[scale=0.8]{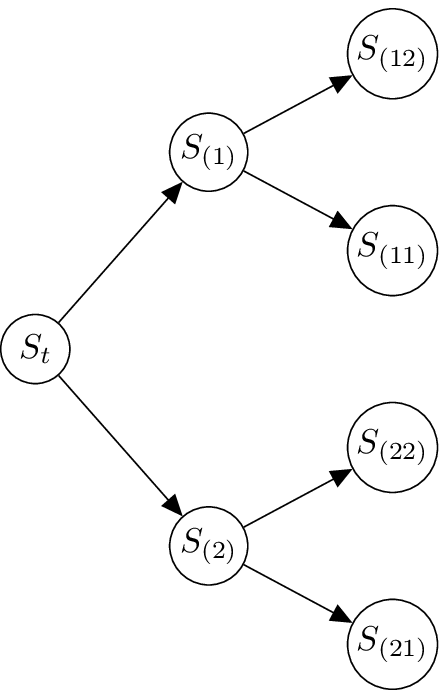}
	\end{center}
    \caption{Illustration of all possible policies up to two time steps in the future when $|U| = 2$. The state at the current time step is denoted by $S_t$. Additionally, each branch of the tree corresponds to a possible policy, and each node $S_I$ is indexed by a multi-index (e.g. $I=(12)$) representing the sequence of actions that led to this state. This should make it clear that for one time step in the future, there are $|U|$ possible policies, after two time steps there are $|U|$ times more policies, and so on until the time-horizon $T$ where there are a total of $|U|^T$ possible policies, i.e., the number of possible policies grows exponentially with the number of time steps for which the agent tries to plan.}
    \label{fig:exp_class}
\end{figure}

To show that this exponential explosion is not only a theoretical problem and also appears in practice, we modified the $DEMO\_MDP\_maze.m$ of the SPM\footnote{Statistical parametric mapping (SPM) is a software package created by the Wellcome Department of Imaging Neuroscience at University College London, which was initially developed to carry out statistical analyses of functional neuroimaging data. Today, SPM also contains MatLab simulations implementing active inference agents (among other things), c.f. \url{https://www.fil.ion.ucl.ac.uk/spm/}. } package in two ways. First, we allowed the agent to plan $N$ time steps in the future as follows:
\begin{Verbatim}[commandchars=\\\{\},codes={\catcode`\$=3\catcode`\^=7\catcode`\_=8}]
\PYG{c}{\PYGZpc{} V = Allowable policies (N moves into the future), nu = number of actions}
\PYG{n}{V} \PYG{p}{=} \PYG{p}{[];}
\PYG{k}{for} \PYG{n}{i1} \PYG{p}{=} \PYG{l+m+mi}{1}\PYG{p}{:}\PYG{n}{nu}\PYG{p}{,} \PYG{n}{i2} \PYG{p}{=} \PYG{l+m+mi}{1}\PYG{p}{:}\PYG{n}{nu}\PYG{p}{,} \PYG{p}{..,} \PYG{n}{iN} \PYG{p}{=} \PYG{l+m+mi}{1}\PYG{p}{:}\PYG{n}{nu}
      \PYG{n}{V}\PYG{p}{(:,}\PYG{k}{end} \PYG{o}{+} \PYG{l+m+mi}{1}\PYG{p}{)} \PYG{p}{=} \PYG{p}{[}\PYG{n}{i1}\PYG{p}{;}\PYG{n}{i2}\PYG{p}{;} \PYG{p}{..} \PYG{p}{;}\PYG{n}{iN}\PYG{p}{];}
\PYG{k}{end}
\end{Verbatim}

Second, we used the ``tic'' and ``toc'' functions provided by Matlab to track the execution time required to execute the function ``$spm\_maze\_search$'' where the parameters of the model are assumed to be known already by the agent:

\begin{Verbatim}[commandchars=\\\{\},codes={\catcode`\$=3\catcode`\^=7\catcode`\_=8}]
\PYG{n}{tic} \PYG{c}{\PYGZpc{} Start a timer used to evaluate the execution time of spm\PYGZus{}maze\PYGZus{}search}
\PYG{n}{MDP} \PYG{p}{=} \PYG{n}{spm\PYGZus{}maze\PYGZus{}search}\PYG{p}{(}\PYG{n}{mdp}\PYG{p}{,}\PYG{l+m+mi}{8}\PYG{p}{,}\PYG{n}{START}\PYG{p}{,}\PYG{n}{END}\PYG{p}{,}\PYG{l+m+mi}{128}\PYG{p}{,}\PYG{l+m+mi}{1}\PYG{p}{);}
\PYG{n}{toc} \PYG{c}{\PYGZpc{} Display the time elapsed since the last call to tic}
\end{Verbatim}

Figure \ref{fig:exp_explosion} presents the results of our simulations for $N$ from 2 to 6. Under a logarithmic scale on the time axis, the experimental results show that the graph is almost a perfect line, which provides empirical evidence for an exponential time explosion. Note that the simulation for $N = 7$ crashed after trying to allocate an array of 9.5GB (space explosion). In section \ref{sec:MCTS}, we presented an approach proposed to deal with the exponential complexity class that arises during planing, yet is fundamentally similiar to active inference. This effectively means that our paper shows how standard active inference can be made more efficient and scale to longer time horizons.

\begin{figure}[H]
\begin{center}
	\includegraphics[scale=0.8]{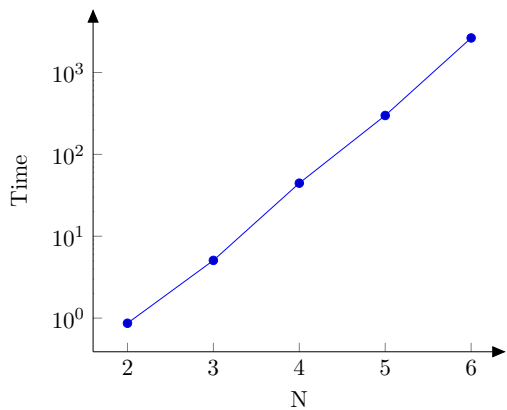}
\end{center}
\caption{This figure shows the time required to execute the function ``$spm\_maze\_search$'' when the agent is allowed to plan $N$ time steps in the future (for $N$ from 2 to 6). A logarithmic scale is used on the time axis.}
\label{fig:exp_explosion}
\end{figure}

\end{document}